\documentclass[12pt]{article}
\usepackage{times}
\usepackage{graphicx}
\usepackage{subfigure}
\usepackage{color}
\usepackage{multirow}
\usepackage[authoryear]{natbib}
\usepackage{amssymb}
\usepackage{amsmath}
\usepackage{amsthm}
\usepackage{amsbsy}
\usepackage{bm}
\usepackage{rotating}
\usepackage{bbm}
\usepackage{latexsym}

\newcommand{\bE}{\mathbb{E}}

\newcommand{\aL}{\mathcal{A}}	

\newcommand{\oT}{\mathcal{T}}
\newcommand{\vL}{\mathcal{V}}
\newtheorem{theorem}{Theorem}
\newtheorem{lemma}[theorem]{Lemma}
\newtheorem{proposition}[theorem]{Proposition}
\newtheorem{assumption}{Assumption}
\newtheorem*{remark}{Remark}

\newcommand{\captionfonts}{\normalsize}

\makeatletter
\long\def\@makecaption#1#2{%
  \vskip\abovecaptionskip
  \sbox\@tempboxa{{\captionfonts #1: #2}}%
  \ifdim \wd\@tempboxa >\hsize
    {\captionfonts #1: #2\par}
  \else
    \hbox to\hsize{\hfil\box\@tempboxa\hfil}%
  \fi
  \vskip\belowcaptionskip}
\makeatother

\DeclareMathOperator*{\argmin}{argmin}
\DeclareMathOperator*{\argmax}{argmax}

\begin{document}
\hspace{13.9cm}

\ \vspace{20mm}\\

{\LARGE Online Markov decision processes with policy iteration}

\ \\
{\bf \large Yao Ma$^{\displaystyle 1}$
, Hao Zhang$^{\displaystyle 1}$, Masashi Sugiyama$^{\displaystyle 2}$}\\
{$^{\displaystyle 1}$Tokyo Institute of Technology,
2-12-1 O-okayama, Meguro, Tokyo, 152-8552, Japan.}\\
{$^{\displaystyle 2}$The University of Tokyo,
7-3-1 Hongo, Bunkyo, Tokyo, 113-0033, Japan.}


\thispagestyle{empty}
\markboth{}{NC instructions}
\ \vspace{-0mm}\\
%
\begin{center} {\bf Abstract} \end{center}
The \emph{online Markov decision process} (MDP)
is a generalization of the classical Markov decision process
that incorporates changing reward functions.
In this paper, we propose practical online MDP algorithms
with \emph{policy iteration}
and theoretically establish a sublinear regret bound.
A notable advantage of the proposed algorithm is that
it can be easily combined with function approximation,
and thus large and possibly continuous state spaces can be efficiently handled.
Through experiments, we demonstrate the usefulness of the proposed algorithm.

\section{Introduction}

A generalization of the classical shortest path problem in graph theory,
called the \emph{stochastic shortest path} problem \citep*{NDPbook09},
considers a probability distribution over all possible next nodes.
A standard way to solve the stochastic shortest path problem 
is to formulate it as a \emph{Markov decision process} (MDP)
and find a policy that maximizes the cumulative reward over the path. 
In the MDP problem, the agent chooses the best action according to the current state and
moves to the next state following the Markovian dynamics.
A fixed  reward function assigns a reward value to each state-action pair.

A generalization of MDP, called the \emph{online MDP},
considers the situation where the reward function changes over time. At each time step, the learning agent decides the strategy of choosing actions by using the knowledge of past reward functions. Then, the current reward function which is chosen by the environment is revealed to the agent after observing its behavior. The goal of online MDP is to minimize the \emph{regret} against the best offline policy,
which is the optimal fixed policy in hindsight. 
We expect that the regret vanishes as the time step $T$ tends to infinity,
implying that the agent can behave as well as the best offline policy asymptotically.

Many online problems can be solved as online MDP problems. By setting the optimization variables as the state, the online MDP algorithm chooses the change of variables (action) which performs reasonably well in a non-stationary environment. \citet{MOOR09} presented several typical online problems which can be formulated as online MDP perfectly, e.g. paging, $k$-server, metrical task system and stochastic inventory control.

The online MDP problem was first introduced by 
\citet{NIPS05,MOOR09} and an expert-based MDP algorithm (MDP-E) was proposed, which was shown to achieve regret $O(\sqrt{T}|A|)$ ($|A|$ is the cardinality of action space) by placing an expert algorithm on every state.
Furthermore, the MDP-E algorithm was proved to achieve regret $O(L^{2}\sqrt{T\log|A|})$ for online MDP problems with $L$-layered state space \citep*{Neu10b}.
However, the MDP-E algorithm is not computationally feasible for
problems with large state space,
since it needs to put the expert algorithm on every state.

Another online MDP algorithm called the \emph{lazy follow-the-perturbed-leader} (lazy-FPL) \citep*{Yu09} follows the main idea of the FPL algorithm which solves the Bellman equation using the average reward function. The ''lazy'' behavior of the lazy-FPL algorithm divides the time horizon into short periods and the policy is only updated at the end of each period. This lazy-FPL algorithm was proved to achieve sublinear regret $O(T^{3/4+\epsilon}\log{T(|S|+|A|)|A|^{2}})$ for $\epsilon\in(0,1/3)$. 

Similarly to lazy-FPL,
the \emph{online relative entropy policy search} (O-REPS) algorithm \citep*{Zimin13} also requires
to solve an optimization problem at the end of each time step.
It was shown that the O-REPS algorithm achieves regret
$O(L\sqrt{T\log(|S||A|/L)})$ for online MDP problems with $L$-layer state space.
Thus, the regret bound of O-REPS is much sharper than those for the MDP-E algorithm
when $L$ is large.
However, O-REPS requires the length of time horizon $T$ to be finite,
because the step size for parameter update needs to be set as a function of $T$.
Therefore, it cannot be directly extended to problems with infinite time horizon.
By introducing the stationary occupation measure, \citet{Dick14} proposed the \emph{mirror descent with approximation projections} algorithm, which formulate the online MDP problem as online linear optimization. Their theoretical results show that the regret is bounded by $O(\sqrt{T})$ where the finite state space assumption is essential.\citet{Yu09}, \citet{Yasin13},
and \citet{Neu12} considered even more challenging online MDP problems
under unknown or changing transition dynamics.

Recently, \citet{YaoECML14} proposed the \emph{online policy gradient} (OPG) algorithm
for online MDP problems with continuous state and action spaces,
and it was proved to achieve regret $O(\sqrt{T})$  under the concavity assumption
about the expected average reward function.
Although the OPG algorithm is natural and efficient for continuous problems,
the concavity assumption may not be realistic in practice.


The aim of this paper is to develop a novel algorithm for solving online MDPs
that is computationally efficient and performs well in problems with large state spaces.
More specifically, we propose a policy iteration algorithm for online MDPs (OMDP-PI), which has a close form update rule at each time step. We prove that
our proposed algorithm achieves a sublinear regret with respect to a policy set. We further extend the proposed OMDP-PI algorithm with linear function approximation, which is essential for large (continuous) state space.  


The remainder of this paper is organized as follows.
In Section~2, we give the formal definition of online MDPs.
In Section~3 we give the details of the proposed algorithm and analyze its regret. A generalization of the proposed algorithm with linear function approximation is also analyzed here.
In Section~4, we present a discussion on solving online MDPs with stochastic iteration.
In Section~5, we demonstrate the performance of the proposed algorithm
in simulation experiments.
In Section~6, we compare the related works with the proposed algorithm.
Finally, in Section~7, we conclude the paper.

\section{Problem definition and preliminaries}
In this section, we present the formal definition and involved preliminaries of the
online MDP problem.

\subsection{Online Markov decision process}
First, we formulate the problem of online MDP learning \citep*{NIPS05,MOOR09}
specified by $\{S,A,P,[r_{t}]_{t=1,\ldots,T}\}$, where
\begin{itemize}
\item
$S$ is the state space, and $|S|$ is the cardinality of state space.
\item 
$A$ is the action space, and $|A|$ is the cardinality of action space. 
\item
$P:S\times S\times A\rightarrow [0,1]$ is the transition probability, where $p(\bm{s}'|\bm{s},\bm{a})$ gives the conditional probability of next state $\bm{s}'$ by taking action $\bm{a}$ at state $\bm{s}$. We assume that the transition probability is available for the agent.
\item
$r_{1},\ldots,r_{T}$ is the reward function sequence, and only $r_{1},\ldots,r_{t}$ are observed at time step $1\le t\le T$.
\end{itemize}
At the end of each time step $t=1,\ldots,T$, trajectory $\bm{h}_{t}$ is observed:
\begin{align*}
\bm{h}_{t}&=\{\bm{s}_{1},\bm{a}_{1},r_{1}(\bm{s},\bm{a})
,\ldots,
\bm{s}_{t},\bm{a}_{t},r_{t}(\bm{s},\bm{a})\}.
\end{align*}

The objective of an online MDP algorithm is to produce a strategy of choosing an action at time step $t$ after observing $\bm{h}_{t}$. More specifically, let $\pi_{t}(\bm{a}|\bm{s}),\forall \bm{s}\in S,\bm{a}\in A$ 
be a stochastic time-dependent policy,
which is the conditional probability of action $\bm{a}$
to be taken at state $\bm{s}$ at time step $t$.

An online MDP algorithm $\aL$ learns
a time dependent policy that maximizes the expected cumulative rewards:
\begin{equation*}
R_{\aL}(T)=\sum_{t=1}^{T}\bE_{\pi_{t}}\left[
r_{t}(\bm{s}_{t},\bm{a}_{t})|\aL
\right],
\end{equation*}
where $\pi_{1},\ldots,\pi_{T}$ is the policy sequence generated by algorithm $\aL$ and $\bE_{\pi_{t}}[\cdot|\aL]$ denotes the expectation over the joint state-action distribution $p_{t}(\bm{s},\bm{a}|\aL)=p(\bm{s}_{t}=\bm{s}|\aL)\pi_{t}(\bm{a}|\bm{s})$ at time step $t$.

However, given that no information is available about future reward functions,
directly analyzing the expected cumulative rewards is not meaningful.
Here, in the same way as standard online learning literature \citep*{Cesa-Bianchi},
we consider the \emph{regret} against the best offline time independent policy $\pi^{*}$ in the policy set $\Pi$:
\begin{equation*}
L_{\aL}(T)=R_{\pi^{*}}(T)-R_{\aL}(T).
\end{equation*}
More precisely, $R_{\pi^{*}}(T)$ is the return of $\pi^{*}$
the best offline time independent policy:
\begin{align*}
R_{\pi^{*}}(T)&=\bE_{\pi^{*}}\left[\sum_{t=1}^{T}r_{t}(\bm{s}_{t},\bm{a}_{t})\right]
=
\sup_{\pi\in \Pi}\bE_{\pi}\left[\sum_{t=1}^{T}r_{t}(\bm{s}_{t},\bm{a}_{t})\right],
\end{align*}
where $\bE_{\pi}[\cdot]$ denotes the expectation over the state-action joint distribution given policy $\pi$.
Note that the regret we consider here is different from previous literature \citep*{NIPS05,MOOR09,Zimin13,Dick14}:
we compare the performance of algorithm $\aL$ against the best offline policy
\emph{within} a specific policy set $\Pi$.
Namely instead of the best deterministic greedy policy, we consider a set of ``efficient'' policies,
e.g., Gibbs policies with all possible parameters.

We expect that the regret $L_{\aL}(T)$ is sublinear with respect to $T$,
which means that the regret tends to zero as $T$ tends to infinity
and thus algorithm $\aL$ performs as well as the best offline policy $\pi^{*}$ asymptotically.

\subsection{Preliminaries}


Next, we introduce some necessary notions for discussing
online MDP problems. First, we show some criterion for evaluating the performance of any stochastic policy. 
For any policy $\pi\in\Pi$, the expected average reward $\rho(\pi)$ is defined as
\begin{align*}
\rho_{r}(\pi)&=\bE_{\bm{s}\sim d_{\pi}(\bm{s}),\bm{a}\sim\pi}\left[r(\bm{s},\bm{a})\right]\\
&=\sum_{\bm{s}\in S}\sum_{\bm{a}\in A}d_{\pi}(\bm{s})\pi(\bm{a}|\bm{s})r(\bm{s},\bm{a}),
\end{align*}
where $d_{\pi}(\bm{s})$ is the stationary state distribution that satisfies
\begin{equation*}
d_{\pi}(\bm{s}')=\sum_{\bm{s}\in S}d_{\pi}(\bm{s})\sum_{\bm{a}\in A}\pi(\bm{a}|\bm{s})p(\bm{s}'|\bm{s},\bm{a}).
\end{equation*}
It has been shown that every ergodic MDP has a unique stationary state distribution. In this paper, we assume that for all $\pi\in\Pi$ 
the target MDP is ergodic.

Another way to evaluate the policy is to define the value function as
\begin{equation*}
\vL_{r}^{\pi}(\bm{s})=\bE_{\pi}\left[\sum_{i=1}^{\infty}(r(\bm{s}_{i},\bm{a}_{i})-\rho_{r}(\pi))|\bm{s}_{1}=\bm{s}\right],
\end{equation*}
For any arbitrary reward function $r(\bm{s},\bm{a})$ and transition probability $p(\bm{s}'|\bm{s},\bm{a})$, there exist at least one optimal policy $\pi^{+}\in \Pi$ such that
\begin{equation*}
\vL_{r}^{\pi^{+}}(\bm{s})\geq \vL_{r}^{\pi}(\bm{s}), \forall \pi\in\Pi,\bm{s}\in S,
\end{equation*} 
\begin{equation*}
\rho_{r}(\pi^{+})\geq\rho_{r}(\pi),\forall\pi\in\Pi.
\end{equation*}

Similarly, the state-action function is defined as
\begin{equation*}
Q_{r}^{\pi}(\bm{s},\bm{a})=\bE_{\pi}\left[\sum_{i=1}^{\infty}(r(\bm{s}_{i},\bm{a}_{i})-\rho_{r}(\pi))|\bm{s}_{1}=\bm{s},\bm{a}_{1}=\bm{a}\right].
\end{equation*}
Since the optimal value function leads to the optimal policy, MDP is often solved by deriving the optimal value function \citep*{Sutton98}. So far, various efficient methods for approximating the optimal value function have been proposed. However, these algorithms were not proved to converge to the value function corresponding to the optimal deterministic policy. For this reason, in this paper we only consider the stochastic policy, since the convergence guarantee is provided \citep*{Tsi99}. 

\section{Online MDPs with policy iteration}

In this section, we introduce the proposed method for online MDPs. The key idea of the proposed algorithm is motivated by the Lazy FPL algorithm by \citet{Yu09}, which performs linear programming to obtain the `leader' policy. As \citet{Yu09} pointed out, solving linear programming may not be appropriate for problems with large (continuous) state space. For this reason, we employ a policy iteration type method together with a stochastic policy in our proposed method.

\subsection{Algorithm}

Firstly, we define the policy improvement operator $\Gamma:\pi(\bm{a}|\bm{s})=\Gamma(r(\bm{s},\bm{a}),V(\bm{s}))$, where $r(\bm{s},\bm{a})$ is an arbitrary reward function, $V(\bm{s})$ is an arbitrary value function. Below we use $\Gamma(r,V)$ instead of $\Gamma(r(\bm{s},\bm{a}),V(\bm{s}))$ for notational simplicity. Now we introduce two assumptions on the defined operator $\Gamma$.
\begin{assumption}
\label{AssumptionPolicy}
For an arbitrary reward function $r$ and two arbitrary value functions $V_{1}(\bm{s})$ and $V_{2}(\bm{s})$, the policies $\pi_{1}=\Gamma(r,V_{1})$ and $\pi_{2}=\Gamma(r,V_{2})$ satisfy
\begin{align*}
\|\pi_{1}(\bm{s},\cdot)-\pi_{2}(\bm{s},\cdot)\|_{1}
\leq \xi\|V_{1}(\cdot)-V_{2}(\cdot)\|_{\infty},
\end{align*}
where $\xi>0$ is the Lipschitz constant depending on the specific policy model. $\|\cdot\|_{1}$ denotes the $L_{1}$ norm, $\|\cdot\|_{\infty}$ denotes the infinity norm in this paper.
\end{assumption}
\begin{assumption}
\label{AssumptionPolicyr}
For an arbitrary value function $V(\bm{s})$ and two arbitrary reward functions $r(\bm{s},\bm{a})$ and $r'(\bm{s},\bm{a})$, the policies $\pi=\Gamma(r,V)$ and $\pi'=\Gamma(r',V)$ satisfy
\begin{align*}
\|\pi(\bm{s},\cdot)-\pi'(\bm{s},\cdot)\|_{1}
\leq& \xi\|r(\bm{s},\cdot)-r'(\bm{s},\cdot)\|_{\infty},
\end{align*}
\end{assumption}

The \emph{Gibbs policy} is a popular model which was demonstrated to work well:
\begin{equation*}
\pi(\bm{a}|\bm{s})=\frac{\exp{\frac{1}{\kappa}\left(r(\bm{s},\bm{a})+\sum_{\bm{s}'\in S}p(\bm{s}'|\bm{s},\bm{a})V(s')\right))}}{\sum_{\bm{a}'\in A}\exp{\frac{1}{\kappa}\left(r(\bm{s},\bm{a}')+\sum_{\bm{s}'\in S}p(\bm{s}'|\bm{s},\bm{a}')V(s')\right)}},
\end{equation*}
where $\kappa$ is the exploration parameter. We can show that the Gibbs policy satisfies Assumption~\ref{AssumptionPolicy} and Assumption~\ref{AssumptionPolicyr} (the proofs are provided in Appendix~\ref{app:Gibbspolicy}).

Throughout this paper, we only consider stochastic policies that satisfy the above two assumptions. Let $\Pi$ be the set of policies generated by the operator $\Gamma$. Then our proposed \emph{online MDP with policy iteration (OMDP-PI) algorithm} is given as follows:

\begin{itemize}
\item Initialize the value function $V_{0}(\bm{s})=0,\forall \bm{s}\in S$.
\item for $t=1,\ldots,\infty$
\begin{enumerate}
\item Observe the current state $\bm{s}_{t}=\bm{s}$.
\item Improve the policy as $\pi_{t}=\Gamma(\hat{r}_{t-1},V_{t-1})$, where 
\begin{equation*}
\hat{r}_{t-1}(\bm{s},\bm{a})=\frac{1}{t-1}\sum_{k=1}^{t-1}r_{k}(\bm{s},\bm{a}).
\end{equation*}
\item Take action $\bm{a}_{t}=\bm{a}$ by following $\pi_{t}$.
\item The reward function $r_{t}(\bm{s},\bm{a})$ is revealed.
\item Update the value function according to
				\begin{equation}
				\label{TDupdaterule}
					 V_{t}(\bm{s})=(1-\gamma_{t})V_{t-1}(\bm{s})+\gamma_{t}\vL_{r_{t}}^{\pi_{t}}(\bm{s}),
				\end{equation}
				where the step size is $\gamma_{t}=1/t$.
\end{enumerate}
\end{itemize}

It is well known \citep*{Sutton98} that the value function satisfies
\begin{equation*}
\vL_{r}^{\pi}(\bm{s})=\bE_{\pi}\left[r(\bm{s},\bm{a})-\rho_{r}(\pi)+\sum_{\bm{s}'\in S}p(\bm{s}'|\bm{s},\bm{a})\vL_{r}^{\pi}(\bm{s}')\right].
\end{equation*}
The above equation can be rewritten in matrix form as
\begin{equation}
\label{reccu}
\vL_{r}^{\pi}=R(\pi)-\bm{e}_{|S|}\rho_{r}(\pi)+P^{\pi}\vL_{r}^{\pi},
\end{equation}
where $\vL_{r}^{\pi}$ is the $|S|$-dimensional column vector whose $\bm{s}$th element is $\vL_{r}^{\pi}(\bm{s})$. $R(\pi)$ is the $|S|$-dimensional column vector whose $\bm{s}$th element is $\sum_{\bm{a}\in A}\pi(\bm{a}|\bm{s})r(\bm{s},\bm{a})$. $P^{\pi}$ is the transition matrix induced by the policy $\pi$, whose $\bm{s}\bm{s}'$th element is $p^{\pi}(\bm{s}|\bm{s}')=\sum_{\bm{a}\in A}\pi(\bm{a}|\bm{s})p(\bm{s}'|\bm{s},\bm{a})$. $\bm{e}_{|S|}$ is the $|S|$-dimensional column vector with all ones. It is well known \citep*{Sutton98} that the above equation has no unique solution. Here we introduce the following constraint on the value function:
\begin{align*}
\bE_{\bm{s}\sim d_{\pi}(\bm{s})}[\vL_{r}^{\pi}(\bm{s})]=\bE_{\bm{s}\sim d_{\pi}(\bm{s}),\bm{a}\sim\pi}\left[\sum_{i=1}^{\infty}(r(\bm{s},\bm{a})-\rho_{r}(\pi))\right]=0.
\end{align*}
By this constraint, the solution of Equ.\eqref{reccu} becomes unique and satisfies
\begin{equation}
\label{solution}
\vL_{r}^{\pi}=R(\pi)-\bm{e}_{|S|}\rho_{r}(\pi)+P^{\pi}\vL_{r}^{\pi}-\bm{e}_{|S|}d_{\pi}^{\top}\vL_{r}^{\pi},
\end{equation}
where $d_{\pi}$ is the $|S|$-dimensional column vector whose $\bm{s}$th element is $d_{\pi}(\bm{s})$.

Then the update rule \eqref{TDupdaterule} can be expressed in closed form as
\begin{equation*}
V_{t}=(1-\gamma_{t})V_{t-1}+\gamma_{t}(\bm{I}_{|S|}-P^{\pi_{t}}+\bm{e}_{|S|}d_{\pi_{t}}^{\top})^{-1}(R_{t}(\pi_{t})-\bm{e}_{|S|}\rho_{r_{t}}(\pi_{t})).
\end{equation*}
Since the stationary distribution can be obtain by the eigenvector corresponding to the unit eigenvalue, we can calculate $\rho_{r_{t}}(\pi_{t})$ directly. Then, $V_{t}(\bm{s})$ can be obtained directly without solving an optimization problem when the state space is not large (continuous). In the following sections, we will introduce an approximation method to handle large (continuous) state space problems. 

\subsection{Regret analysis}
In this section, we provide a regret analysis for the proposed OMDP-PI algorithm. Firstly, we introduce several essential assumptions involved in the proof. Similarly to the previous works \citep*{NIPS05,MOOR09,Yu09,Neu10a,Neu14,YaoECML14}, we assume the following conditions.
\begin{assumption}
\label{mix}
For all $\pi\in\Pi$, there exist a positive constant $\tau$ such that two arbitrary state distributions $d(\bm{s})$ and $d'(\bm{s})$ satisfy
\begin{equation*}
\sum_{\bm{s}\in S}\sum_{\bm{s}'\in S}|d(\bm{s})-d'(\bm{s})|p^{\pi}(\bm{s}'|\bm{s})\leq e^{-1/\tau}\sum_{\bm{s}\in S}|d(\bm{s})-d'(\bm{s})|.
\end{equation*}
\end{assumption}
\begin{assumption}
The reward functions satisfy
\begin{equation*}
r_{t}(\bm{s},\bm{a})\in[0,1],\forall\bm{s}\in S,\forall\bm{a}\in A,\forall t=1,\ldots,T.
\end{equation*}
\end{assumption}

Under these assumptions, the regret of the OMDP-PI algorithm for a policy set $\Pi$ is bounded as follows:
\begin{theorem}
\label{theorem1}
After $T$ time steps, the regret against the best offline time independent policy of the OMDP-PI algorithm is bounded as
\begin{align*}
L_{\mathrm{OMDP-PI}}(T)&\leq \frac{2-e^{-1/\tau}}{1-e^{-1/\tau}}C\xi T^{C_{v}}+\left(\frac{6\tau\xi(2-e^{-1/\tau})}{1-e^{-1/\tau}}+2\tau^{3}\right)\ln{T}\\
&~~~+\left(\frac{6\tau\xi(2-e^{-1/\tau})}{1-e^{-1/\tau}}+2\tau^{3}+2\tau^{3}e^{\tau+2}+4\tau\right),
\end{align*}
where $C=6\tau(2-C_{v}+\frac{1}{C_{v}}+\frac{1-C_{v}}{1+C_{v}})$, $C_{v}=\xi C_{\pi}$, and $C_{\pi}$ is a positive constant such that for all $\pi_{1},\pi_{2}\in\Pi$,
\begin{equation*}
\|\vL_{r}^{\pi_{1}}-\vL_{r}^{\pi_{2}}\|_{\infty}\leq
C_{\pi}\|\pi_{1}-\pi_{2}\|_{1}.
\end{equation*} 
\end{theorem}
The existence of $C_{\pi}$ is proved in Appendix~\ref{Appendix:contraction}.
\begin{remark}
The regret bound in Theorem~\ref{theorem1} is sublinear with respect to $T$ when $C_{v}<1$. However, the quality of the policy is limited when $C_{v}$ is small. Since the smaller the constant $C_{v}$ is, the poorer the performance of the best offline policy is. In an extreme case, where all the policies in the set $\Pi$ perform equally, when $C_{v}=0$.
\end{remark}
To prove the claimed result in Theorem~\ref{theorem1}, we decompose the regret into three parts in the same way as previous works \citep*{NIPS05,MOOR09,Yasin13,YaoECML14}:
\begin{align*}
L_{\aL}(T)&=\left(\bE_{\pi^{*}}\left[\sum_{t=1}^{T}r_{t}(\bm{s}_{t},\bm{a}_{t})\right]-\sum_{t=1}^{T}\rho_{r_t}(\pi^{*})\right)+\left(\sum_{t=1}^{T}\rho_{r_t}(\pi^{*})-\sum_{t=1}^{T}\rho_{r_t}(\pi_{t})\right)\\
&~~~+\left(\sum_{t=1}^{T}\rho_{r_t}(\pi_{t})-\bE_{\pi_{t}}\left[\sum_{t=1}^{T}r_{t}(\bm{s}_{t},\bm{a}_{t})\right]\right).
\end{align*}
The first term has been analyzed in previous works \citep*{NIPS05,MOOR09,YaoECML14}, which is bounded as
\begin{equation*}
\bE_{\pi^{*}}\left[\sum_{t=1}^{T}r_{t}(\bm{s}_{t},\bm{a}_{t})\right]-\sum_{t=1}^{T}\rho_{r_t}(\pi^{*})\leq 2\tau.
\end{equation*}
Below, we bound the second and the third terms in Lemma~\ref{lemma2} and Lemma~\ref{lemma3} which are proved in Appendix~\ref{Appendix:lemma2} and Appendix~\ref{Appendix:lemma3}.
\begin{lemma}
\label{lemma2}
After $T$ time steps, the policy sequence $\pi_{1},\ldots,\pi_{T}$ given by OMDP-PI and the best offline policy $\pi^{*}\in\Pi$ satisfy
\begin{equation*}
\sum_{t=1}^{T}\rho_{r_t}(\pi^{*})-\sum_{t=1}^{T}\rho_{r_t}(\pi_{t})\leq
\frac{2-e^{-1/\tau}}{1-e^{-1/\tau}}\left(C\xi T^{C_{v}}+6\tau \xi\ln{T}+6\tau \xi\right),
\end{equation*}
where $C=6\tau (2-C_v+\frac{1}{C_v}+\frac{1-C_v}{1+C_v})$.
\end{lemma}

\begin{lemma}
\label{lemma3}
After $T$ time steps, the policy sequence $\pi_{1},\ldots,\pi_{T}$ given by OMDP-PI satisfies
\begin{equation*}
\sum_{t=1}^{T}\rho_{r_t}(\pi_{t})-\bE_{\pi_{t}}\left[\sum_{t=1}^{T}r_{t}(\bm{s}_{t},\bm{a}_{t})\right]\leq 2\tau^{3}\ln{T}+2\tau^{3}+2\tau^{3}e^{(\tau+2)}+2\tau.
\end{equation*}
\end{lemma}

Summarizing these bounds, we can obtain Theorem~\ref{theorem1}.
\subsection{OMDP-PI algorithm with approximation}

When considering large (continuous) state space in online MDP problems, it is essential to apply a function approximation technique. \citet{Tsi99} introduced the linear function approximation of the value function for stochastic policies. A significant benefit of the linear approximation is that the convergence guarantee is provided \citep*{Tsi99}. Below we present their theoretical results for discrete (possibly continuous) state space.

By following the same idea as \citet{Tsi99}, we use the linear approximation of the value function:
\begin{equation*}
\hat{\vL}(\bm{s})=\bm{\theta}^{\top}\bm{\phi}(s),
\end{equation*}
where $\bm{\theta}\in\Theta$ is the approximation parameter, and $\Theta\subset\mathbb{R}^{K}$ is the parameter space, $\bm{\phi}(\bm{s})$ is the basis function. At each time step $t$, the value function $\vL_{r_{t}}^{\pi_{t}}(\bm{s})$ is approximated as follows:
\begin{itemize}
\item for $i=1,2,\ldots$ until convergence
\begin{enumerate}
\item Observe the state $\bm{s}_{i}$.
\item Take action $\bm{a}_{i}$ following $\pi_{t}$.
\item Observe the next state $\bm{s}_{i+1}$ and the reward $r_{t}(\bm{s}_{i},\bm{a}_{i})$
\item Update the approximation parameter as
\begin{equation*}
\bm{\theta}_{i+1}=\bm{\theta}_{i}+\alpha_{t}(r_{t}(\bm{s}_{i},\bm{a}_{i})-\hat{\rho}_{r_{t}}^{\pi_{t}}(i)+\bm{\theta}_{i}^{\top}\bm{\phi}(\bm{s}_{i+1})-\bm{\theta}_{i}^{\top}\bm{\phi}(\bm{s}_{i}))
\end{equation*}
and
\begin{equation*}
\hat{\rho}_{r_{t}}^{\pi_{t}}(i+1)=(1-\alpha_{t})\hat{\rho}_{r_{t}}^{\pi_{t}}(i)+\alpha_{t}r_{t}(\bm{s}_{i},\bm{a}_{i}),
\end{equation*}
where the step size $\alpha_{t}$ satisfies 
\begin{equation*}
\sum_{t=1}^{\infty}\alpha_{t}=\infty~\mathrm{and}~\sum_{t=1}^{\infty}\alpha_{t}^{2}<\infty.
\end{equation*}
\end{enumerate}
\end{itemize}
The approximation parameter was proved to converge to the unique solution of the following equation \citep*{Tsi99}:
\begin{equation}
\label{solution}
\mathcal{P}(R_{t}(\pi_{t})-\bm{e}_{|S|}\rho_{r_{t}}(\pi_{t})+P^{\pi_{t}}\bm{\theta}^{\top}\bm{\phi})=\bm{\theta}^{\top}\bm{\phi},
\end{equation} 
where $R_{t}(\pi_{t})$ is the $|S|$-dimensional column vector whose $\bm{s}$th element is $r_{t}(\bm{s},\pi_{t})=\sum_{\bm{a}\in A}\pi_{t}(\bm{a}|\bm{s})r_{t}(\bm{s},\bm{a})$. $\mathcal{P}$ is the projection operator such that for all $V\in\mathbb{R}^{|S|}$,
\begin{equation*}
\mathcal{P}(V)=\argmin_{\bar{V}\in\{\bm{\theta}^{\top}\bm{\phi}|\bm{\theta}\in\mathbb{R}^{K}\}}\|V-\bar{V}\|_{D^{\pi_{t}}},
\end{equation*}
where $D^{\pi_{t}}$ is the diagonal matrix with the stationary distribution on the diagonal.
It is clear that $\mathcal{P}$ is the projection from the $|S|$-dimensional real space to the space spanned by the basis function. The approximation sequence $\hat{\rho}_{r_{t}}^{\pi_{t}}(1),\hat{\rho}_{r_{t}}^{\pi_{t}}(2),\ldots$ satisfies
\begin{equation*}
\lim_{i\rightarrow\infty}\hat{\rho}_{r_{t}}^{\pi_{t}}(i)\rightarrow\rho_{r_{t}}(\pi_{t}), ~\mathrm{with~probability}~1.
\end{equation*}
Furthermore, by using Theorem 3 in \citet{Tsi99}, the approximation error is bounded as
\begin{equation*}
\|(\bm{I}_{|S|}-\bm{e}_{|S|}d_{\pi_{t}}^{\top})\bm{\theta}^{*\top}_{t}\bm{\phi}-\vL_{r_{t}}^{\pi_{t}}\|_{D_{\pi_{t}}}\leq\frac{1}{\sqrt{1-e^{-2/\tau}}}\inf_{\bm{\theta}\in\mathbb{R}^{K}}\|(\bm{I}_{|S|}-\bm{e}_{|S|}d_{\pi_{t}}^{\top})\bm{\theta}^{\top}\bm{\phi}-\vL_{r_{t}}^{\pi_{t}}\|_{D_{\pi_{t}}},
\end{equation*}
where $\bm{\theta}_{t}^{*}$ is the unique solution to Eq.\eqref{solution} at time step $t$. We observe that the approximation error is zero when the linear approximation model is capable of exactly recovering the true value function. 

\section{Online MDPs with stochastic iteration}

In this section, we introduce a more general framework of our proposed method for online MDPs. More specifically, we extend our algorithm to use \emph{stochastic iteration} \citep{NDPbook09} for policy evaluation together with policy improvement to solve online MDPs. 

A general form of the stochastic iteration algorithm \citep*{Szita02,Csaji08} can be expressed as
\begin{align}
V_{t}(\bm{s})&=(1-\gamma_{t}(\bm{s}))V_{t-1}(\bm{s})
+\gamma_{t}(\bm{s})
\left((H_{t}V_{t-1})(\bm{s})+w_{t}(\bm{s})\right),
\label{stoiteration}
\end{align}
where $V_{t}\in \mathbb{R}^{|S|}$,
$H_{t}:\mathbb{R}^{|S|}\rightarrow\mathbb{R}^{|S|}, \forall t=1,\ldots,T$ is an operator on value functions,
$\gamma_{t}$ is the step size, and $w_{t}(\bm{s})$ is a noise term. Similarly to the Eq.\eqref{stoiteration}, we define the update rule as
\begin{equation}
\label{MDPSI}
V_{t}(\bm{s})=(1-\gamma_{t}(\bm{s}))V_{t-1}(\bm{s})+
\gamma_{t}(\bm{s})((H_{t}^{\pi_{t}}V_{t-1})(\bm{s})+w_{t}(\bm{s})),
\end{equation}
where $\pi_{t}=\Gamma(\hat{r}_{t-1},V_{t-1})$ satisfies Assumption~\ref{AssumptionPolicy} and Assumption~\ref{AssumptionPolicyr}. Note that the update rule~\eqref{MDPSI} is different from standard stochastic iteration~\eqref{stoiteration}, where the operator $H_{t}$ is replaced by the controlled operator $H_{t}^{\pi}$ which the OMDP-PI algorithm uses: $H_{t}^{\pi_{t}}V_{t-1}(\bm{s}
)=\vL_{r_{t}}^{\pi_{t}}(\bm{s})$. Additionally, we require the following assumptions.
\begin{assumption}
\label{AssumptionOper}
The controlled operator $H_{t}^{\pi}$ is a contraction mapping with respect to the value function. This means that, for two arbitrary value functions $V$ and $V'$ and two policies $\pi=\Gamma(r,V)$, $\pi'=\Gamma(r,V')$, there exist a no negative constant $\beta_{t}<1$ such that
\begin{equation*}
\|H_{t}^{\pi}V-H_{t}^{\pi}V'\|\leq \beta_{t}\|V-V'\|,
\end{equation*}
and there exist a fixed function $V_{t}^{*}$ satisfies
\begin{equation*}
H_{t}V_{t}^{*}=V_{t}^{*}.
\end{equation*}
\end{assumption}
\begin{assumption}
For all $t=1,\ldots,T$, the noisy terms $w_{t}(\bm{s})$ satisfy
\begin{equation*}
\bE[w_{t}(\bm{s})]=0~~\mathrm{and}~~\bE[w_{t}^{2}(\bm{s})]<C_{w}<\infty,
\end{equation*}
where $C_{w}$ is a positive constant.
\end{assumption}
\begin{assumption}
The step size $\gamma_{t}$ satisfies
\begin{equation*}
\sum_{t=1}^{\infty}\gamma_{t}=\infty~~\mathrm{and}~~\sum_{t=1}^{\infty}\gamma_{t}^{2}<\infty.
\end{equation*}
\end{assumption}
\begin{assumption}
\label{AssumptionOper1}
The value functions sequence $V_{1}(\bm{s}),\ldots,V_{T}(\bm{s})$ generated by Eq.\eqref{MDPSI} satisfies
\begin{equation*}
\lim_{T\rightarrow\infty}\|V_{T}^{*}-\max_{\pi\in\Pi}\vL_{\hat{r}_{T}}^{\pi}\|_{\infty}=0.
\end{equation*}
\end{assumption}
Then we have the following theorem:
\begin{theorem}
\label{the:si}
If Assumptions 5-8 hold, the value function sequence $V_{1}(\bm{s}),\ldots,V_{T}(\bm{s})$ generated by Eq.\eqref{MDPSI} satisfies
\begin{equation*}
\lim_{T\rightarrow\infty}L_{\aL}(T)=0.
\end{equation*}
\end{theorem}
\begin{proof}
By using Theorem 20 in \citet{Csaji08}, we have
\begin{equation*}
\lim_{T\rightarrow\infty}\|V_{T}-\vL_{\hat{r}_{T}}^{\pi^{*}}\|_{\infty}=0,
\end{equation*}
where $\pi^{*}=\argmax_{\pi\in\Pi}\rho_{\hat{r}_{T}}(\pi)$ is the best offline policy.
Since $\pi_{T+1}=\Gamma(\hat{r}_{T},V_{T})$ and $\pi_{T}$, we obtain
\begin{align*}
\lim_{T\rightarrow\infty}\|\pi_{T}-\pi^{*}\|_{1}&\leq \lim_{T\rightarrow\infty}(\|\pi_{T+1}-\pi^{*}\|_{1}+\|\pi_{T+1}-\pi_{T}\|_{1})\\
&\leq
\lim_{T\rightarrow\infty}(\xi\|V_{T}-\vL_{\hat{r}_{T}}^{\pi^{*}}\|_{\infty}
+\|\pi_{T+1}-\pi_{T}\|_{1})\\
&\leq
\lim_{T\rightarrow\infty}(\xi\|V_{T}-\vL_{\hat{r}_{T}}^{
pi^{*}}\|_{\infty}+\xi\|\hat{r}_{T+1}-\hat{r}_{T}\|_{\infty}+\xi\|V_{T}-V_{T-1}\|_{\infty})\\
&=0.
\end{align*}
In the above derivation we used $\lim_{T\rightarrow\infty}\|V_{T}-V_{T-1}\|_{\infty}=0$, which can be obtained by the update rule \eqref{MDPSI}. The above result shows that the policies generated by the value sequence converges to the best offline policy as $T$ goes to infinity. Hence, the claimed result hold by following the same line as the proof of Theorem~\ref{theorem1}.
\end{proof}

Many popular reinforcement learning algorithms based on value functions such as
the \emph{temporal difference (TD) learning} algorithm \citep*{NDPbook09,Sutton98,Sutton88learningto} and 
the \emph{SARSA} algorithm \citep*{NDPbook09,Sutton98}
can be regarded as stochastic iteration. Theorem~\ref{the:si} shows that any stochastic iteration method that satisfies Assumptions 5-8 could be used to derive an online MDPs algorithm with sublinear regret.

\section{Experiments}
In this section, we experimentally illustrate the behavior of the proposed online algorithm.

\begin{figure*}[t]
\centering
\subfigure[OMDP-PI, t=25000]{
\includegraphics[width=0.35\textwidth]{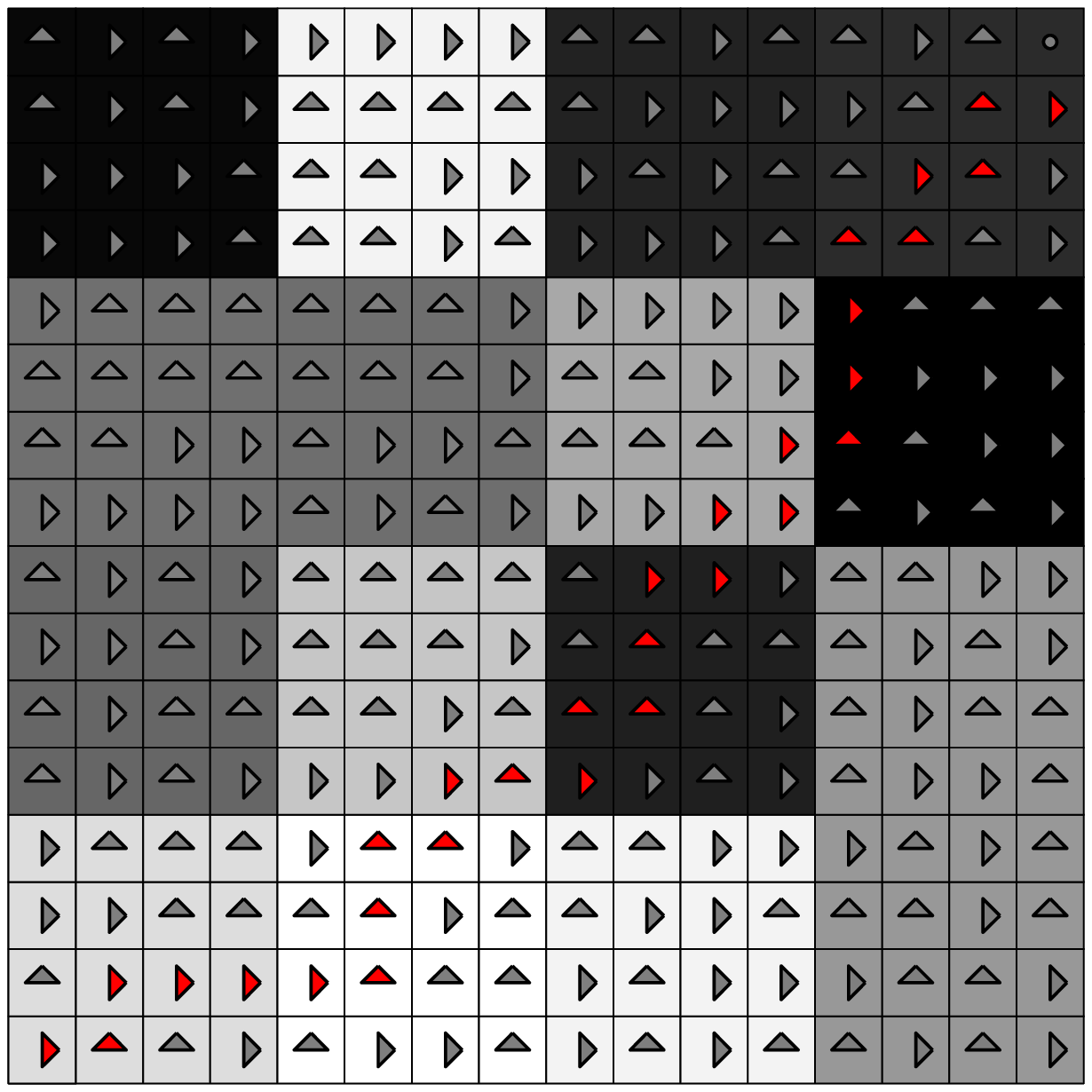}
\label{fig:wTD01}
}
\subfigure[Best Offline, t=25000]{
\includegraphics[width=0.35\textwidth]{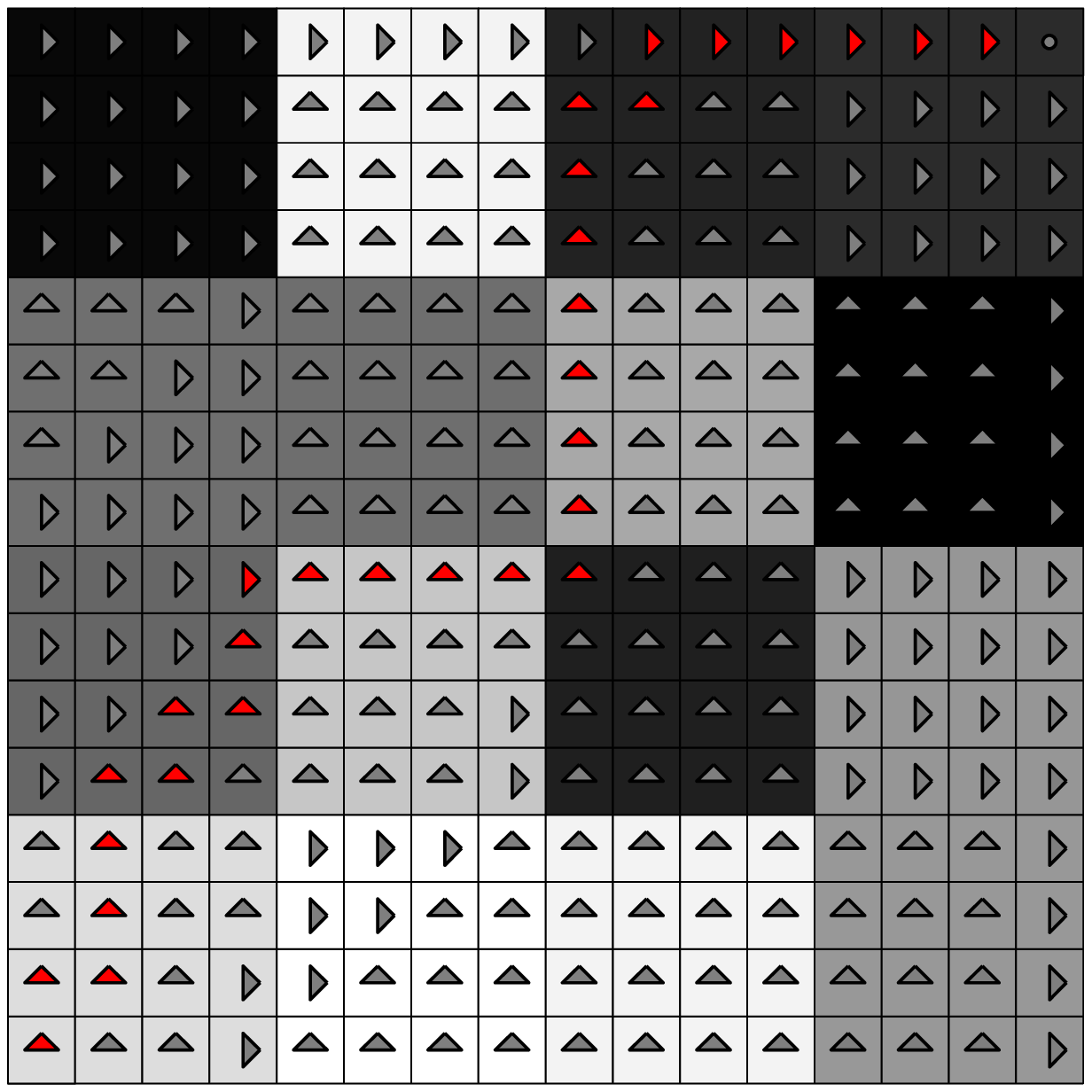}
\label{fig:wOffline01}
}
\subfigure[OMDP-PI, t=50000]{
\includegraphics[width=0.35\textwidth]{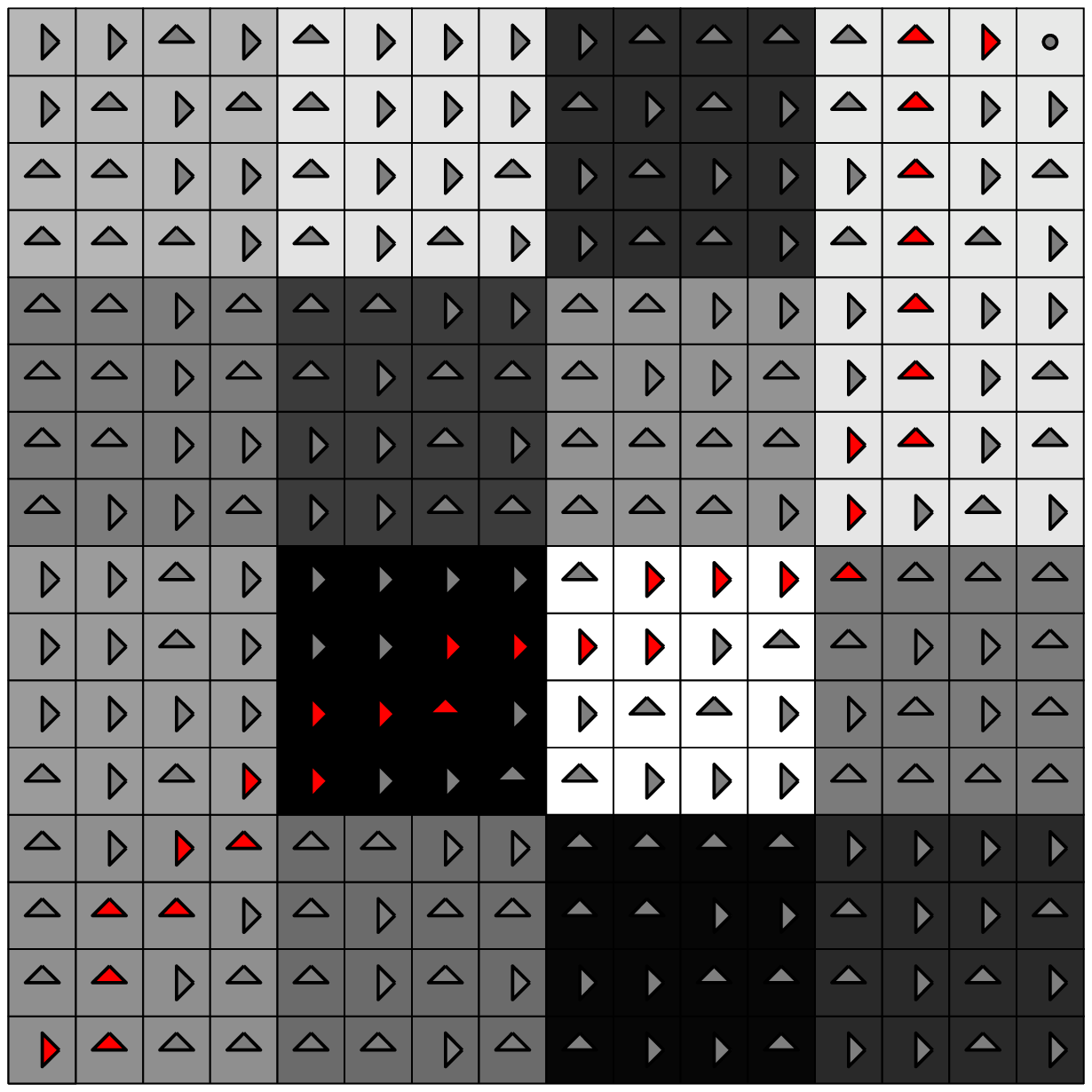}
\label{fig:wTD02}
}
\subfigure[Best Offline, t=50000]{
\includegraphics[width=0.35\textwidth]{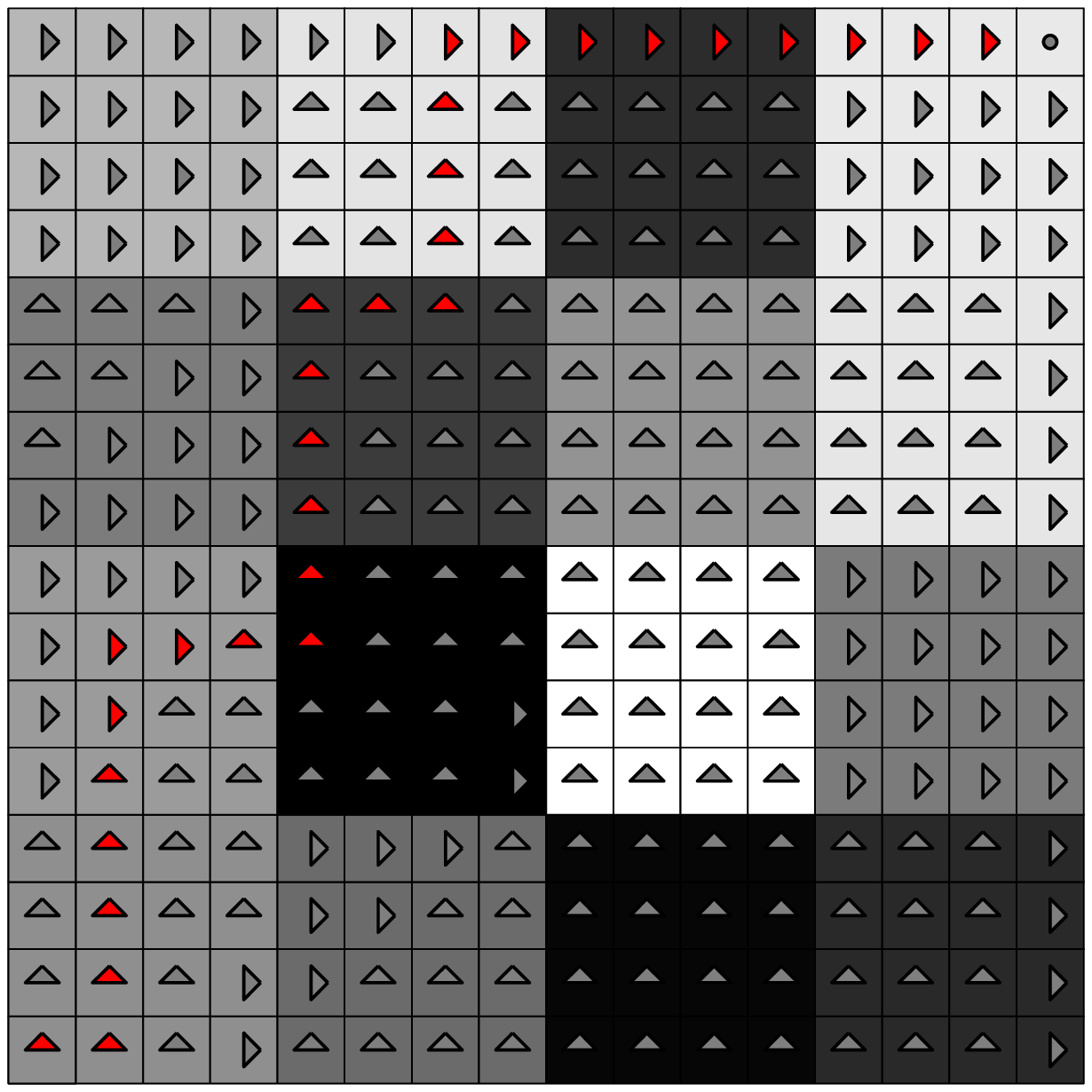}
\label{fig:wOffline02}
}
\subfigure[OMDP-PI, t=75000]{
\includegraphics[width=0.35\textwidth]{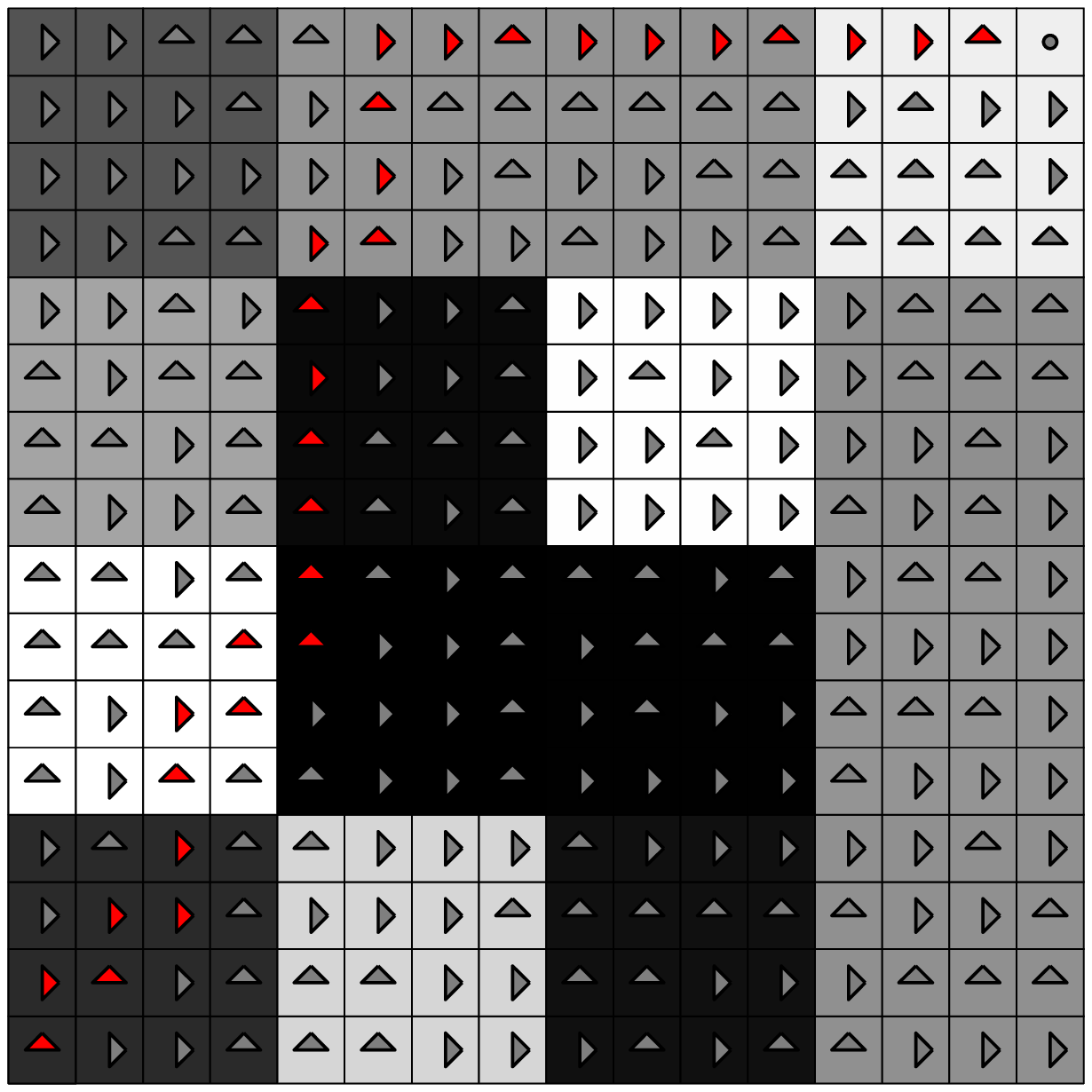}
\label{fig:wTD03}
}
\subfigure[Best Offline, t=75000]{
\includegraphics[width=0.35\textwidth]{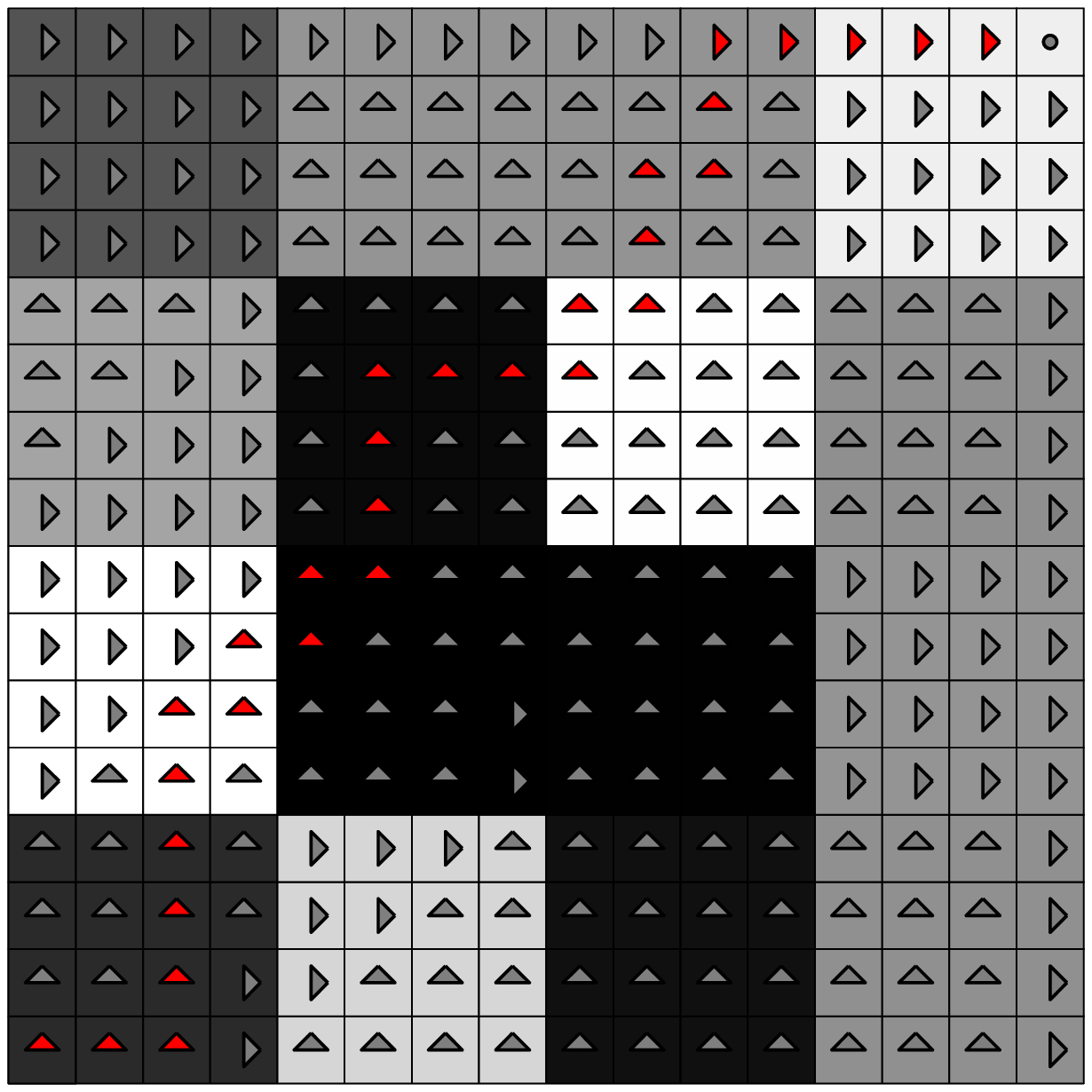}
\label{fig:wOffline03}
}

\caption{Experiments on grid worlds.}
\label{fig:gridworlds}
\end{figure*}

\begin{figure}[t]
\centering
\includegraphics[width=0.85\textwidth]{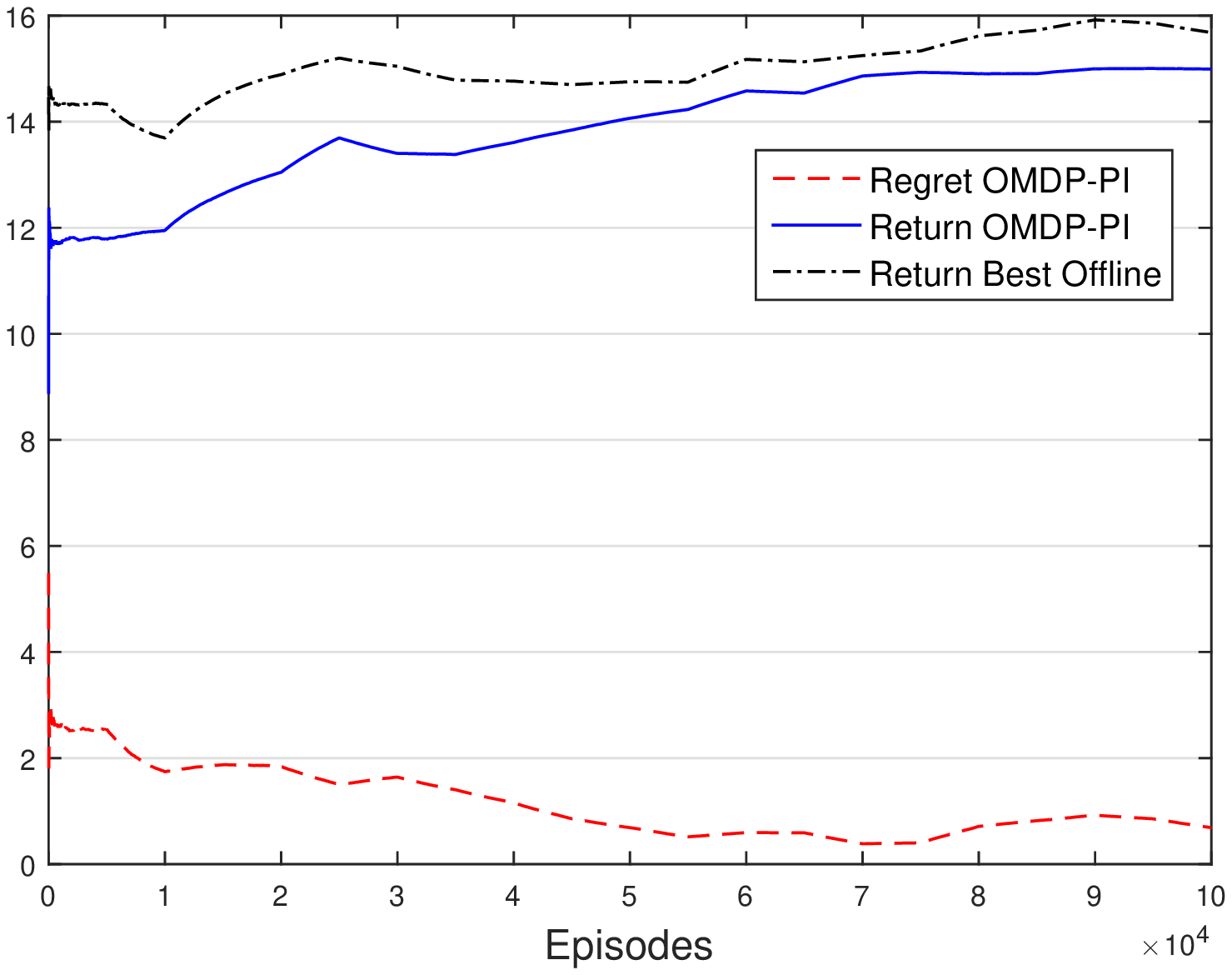}
\caption{Regrets and rewards.}
\label{fig:regret}
\end{figure}

The goal of the \emph{grid world} problem is to let an agent walk in the grid environment
from the start block to the destination block.
We conduct experiments on the grid world based
on the \emph{Inverse Reinforcement Learning} (IRL) toolkit\footnote{http://graphics.stanford.edu/projects/gpirl}\citep*{levine11}.

In our experiments, the grid world has $16\times16$ states with $2$ actions in each state,
which correspond to moving east and north.
Each action has a 30\% chance of moving in the other direction.
The 256 states are further joined into $16$ super-grids,
each of which consists of $4\times4$ states with the same reward.

In each episode, the agent tries to find a trajectory from the south-west corner
to the north-east corner, with the highest cumulative rewards.
In the north or east border states, the agent can only move east or north.
We set $T=100000$ and randomly change the rewards at episodes
$t=1,5001,10001,...,95001$. 
The proposed algorithm and the best offline algorithm
(obtained using the standard MDP solver included in the IRL toolkit) 
are run on the grid world.

Figure~\ref{fig:gridworlds} shows the trajectories found by the offline policy and proposed OMDP-PI algorithm
at episodes $t=25000,50000,75000,100000$.
The darker the state is, the lower average reward it has.
The direction of triangles shows the obtained policies.
The states with red triangles indicate trajectories of the agent.
Figure~\ref{fig:regret} shows the average regret and cumulative reward
as functions of the number of episodes. 

The results in Figure~\ref{fig:regret} show that
the regret of the OMDP-PI algorithms vanishes,
substantiating that our theoretical analysis is valid.

\section{Comparison with previous works}
\begin{itemize}
\item Expert algorithm based methods \citep*{NIPS05,MOOR09,Neu10b,Neu14}: The basic idea of expert algorithm based methods is to put an expert algorithm in every state. By taking a close look at these algorithms, the idea does not take advantage of the state structure of the MDP problem. The OMDP-PI algorithm  can be easily combined with function approximation. Since it is popular to simplify the large state space problem by using the linear span of the state features, the OMDP-PI algorithm is natural to handle the large state space online MDPs.

\item Online linear optimization based methods \citep*{Zimin13,Dick14}: By introducing the stationary occupancy measures over state-action pairs, the online MDP problems can be solved as the online linear optimization problems. The $O(\sqrt{T})$ regret bounds are proved for fixed time horizon online MDPs. More specifically, the step size parameter is optimized by using the length of the time horizon $T$. Moreover, the stationary occupancy measures are defined over finite state and action spaces, and thus it is not clear that whether the state-action probability density function could be learned by using their propose methods without parametrization. The OMDP-PI algorithm with function approximation parameterized the state-action density through the linear model of the value function. 

\item Linear programming based method \citep*{Yu09}: Our OMDP-PI is motivated by the Lazy-FPL algorithm, which solves a linear programming problem at the end of each phase. Instead of obtaining the best policy by the linear programming, the OMDP-PI algorithm obtains the value function of the  current policy which is much more efficient than the linear programming. As we showed in the update rule, the policy evaluation could be performed in $O(|S|^{2.3728639}+|S|^{2}|A|)$ where the matrix inversion could be solved in $O(|S|^{2.3728639})$ \citep*{Le14}.
\end{itemize}

\section{Conclusion and future work}
As a generalization of MDP, online MDP is a promising model which can handle many online problem with guaranteed performance. In this paper, we proposed a policy iteration algorithm with a closed form update rule
for online MDP problems.
We showed that the proposed algorithm 
achieves sublinear regret for a policy set.
A notable fact is that the proposed algorithm is still practical for online MDP problems
with large (continuous) state space.
We showed that the propose algorithm can be easily combined with function approximation with theoretical guarantee. We illustrated the performance of the proposed algorithm through grid-world experiments. 

Our future work will extend the proposed algorithm to the bandit feedback scenario, where the full information of the reward function is not revealed to the agent.
Exploring other stochastic iteration methods
such as the \emph{SARSA} algorithm \citep*{NDPbook09,Sutton98}
and the \emph{value iteration} algorithm \citep*{Csaji08,Szita02}
is also an important future direction.

\bibliographystyle{apa}

\appendix
\section*{Appendix}

\section{Proof of Gibbs policy}
\label{app:Gibbspolicy}
\begin{proof}
The KL divergence of two Gibbs policies $\pi$, $\pi'$ generated by two different value function $V$ and $V'$ is
\begin{align*}
&D(\pi(\cdot|\bm{s})\parallel\pi'(\cdot|\bm{s}))\\
&=
\bE_{\pi}\left[\sum_{\bm{s}'\in S}
\frac{1}{\kappa} p(\bm{s}'|\bm{s},\bm{a})\left(V(\bm{s}')-V'(\bm{s}')\right)\right]\\
&\phantom{=}+
\log{\frac{\sum_{\bm{a}'\in A}\exp{\frac{1}{\kappa}\left(\left(r(\bm{s},\bm{a}')
+\sum_{\bm{s}'\in S}p(\bm{s}'|\bm{s},\bm{a}')V'(\bm{s}')\right)\right)}}{\sum_{\bm{a}'\in A}\exp{\frac{1}{\kappa}\left(\left(r(\bm{s},\bm{a}')
+\sum_{\bm{s}'\in S}p(\bm{s}'|\bm{s},\bm{a}')V(\bm{s}')\right)\right)}}}\\
&=
\bE_{\pi}\left[\frac{1}{\kappa}\sum_{\bm{s}'\in S}
p(\bm{s}'|\bm{s},\bm{a})\left(V(\bm{s}')-V'(\bm{s}')\right)\right]\\
&\phantom{=}+
\log{\frac{\sum_{\bm{a}'\in A}\exp{\frac{1}{\kappa}(r(\bm{s},\bm{a}')+\sum_{\bm{s}'\in S}p(\bm{s}'|\bm{s},\bm{a}')V(\bm{s}')+\sum_{\bm{s}'\in S}p(\bm{s}'|\bm{s},\bm{a}')(V'(\bm{s}')-V(\bm{s}')))}}{\sum_{\bm{a}'\in A}\exp{\frac{1}{\kappa}(r(\bm{s},\bm{a}')+\sum_{\bm{s}'\in S}p(\bm{s}'|\bm{s},\bm{a}')V(\bm{s}'))}}}\\
&=
\bE_{\pi}\left[\frac{1}{\kappa}\sum_{\bm{s}'\in S}
p(\bm{s}'|\bm{s},\bm{a})\left(V(\bm{s}')-V'(\bm{s}')\right)\right]\\
&\phantom{=}+
\log{\frac{\sum_{\bm{a}'\in A}\exp{\frac{1}{\kappa}\left(\left(r(\bm{s},\bm{a}')
+\sum_{\bm{s}'\in S}p(\bm{s}'|\bm{s},\bm{a}')V(\bm{s}')\right)\right)}\exp{\frac{1}{\kappa}\sum_{\bm{s}'\in S}p(\bm{s}'|\bm{s},\bm{a}')(V'(\bm{s}')-V(\bm{s}'))}}{\sum_{\bm{a}'\in A}\exp{\frac{1}{\kappa}\left(\left(r(\bm{s},\bm{a}')
+\sum_{\bm{s}'\in S}p(\bm{s}'|\bm{s},\bm{a}')V(\bm{s}')\right)\right)}}}\\
&=
\bE_{\pi}\left[\frac{1}{\kappa}\sum_{\bm{s}'\in S}
p(\bm{s}'|\bm{s},\bm{a})\left(V(\bm{s}')-V'(\bm{s}')\right)\right]\\
&~~~+\log{\bE_{\pi}\left[\frac{1}{\kappa}\sum_{\bm{s}'\in S}
p(\bm{s}'|\bm{s},\bm{a})\left(V'(\bm{s}')-V(\bm{s}')\right)\right]}\\
&\leq
\frac{\|\frac{1}{\kappa}\sum_{\bm{s}'\in S}p(\bm{s}'|\bm{s},\bm{a})\left(V(\bm{s}')-V'(\bm{s}')\right)\|^{2}_{\infty}}{4}.
\end{align*}
From the Pinsker's inequality, there is
\begin{equation*}
\|\pi(\cdot|\bm{s})-\pi'(\cdot|\bm{s})\|_{1}\leq
\frac{\|V(\bm{s})-V'(\bm{s})\|_{\infty}}{\sqrt{2}\kappa},
\end{equation*}
Similarly, the KL divergence of two Gibbs policies $\pi$, $\pi'$ generated by two different reward function $r(\bm{s},\bm{a})$ and $r'(\bm{s},\bm{a})$ is
\begin{align*}
&D(\pi(\cdot|\bm{s})||\pi'(\cdot|\bm{s}))\\
&=\bE_{\pi}\left[\frac{1}{k}(r(\bm{s},\cdot)-r'(\bm{s},\cdot))\right]\\
&~~~+\log{\frac{\sum_{\bm{a}'\in A}\exp\frac{1}{\kappa}(r'(\bm{s},\bm{a}')+\sum_{\bm{s}'\in S}p(\bm{s}'|\bm{s},\bm{a}')V(\bm{s}'))}{\sum_{\bm{a}'\in A}\exp\frac{1}{\kappa}(r(\bm{s},\bm{a}')+\sum_{\bm{s}'\in S}p(\bm{s}'|\bm{s},\bm{a}')V(\bm{s}'))}}\\
&=
\bE_{\pi}\left[\frac{1}{k}(r(\bm{s},\cdot)-r'(\bm{s},\cdot))\right]\\
&~~~+
\log{\frac{\sum_{\bm{a}'\in A}\exp\frac{1}{\kappa}(r(\bm{s},\bm{a}')+\sum_{\bm{s}'\in S}p(\bm{s}'|\bm{s},\bm{a}')V(\bm{s}')+r'(\bm{s},\bm{a}')-r(\bm{s},\bm{a}'))}{\sum_{\bm{a}'\in A}\exp\frac{1}{\kappa}(r(\bm{s},\bm{a}')+\sum_{\bm{s}'\in S}p(\bm{s}'|\bm{s},\bm{a}')V(\bm{s}'))}}\\
&=\bE_{\pi}\left[\frac{1}{\kappa}(r(\bm{s},\cdot)-r'(\bm{s},\cdot))\right]\\
&~~~+\log{\bE_{\pi}\left[\frac{1}{\kappa}(r'(\bm{s},\cdot)-r(\bm{s},\cdot))\right]}\\
&\leq
\frac{\|\frac{1}{\kappa}(r(\bm{s},\cdot)-r'(\bm{s},\cdot))\|^{2}_{\infty}}{4}.
\end{align*}
From the Pinsker's inequality, we can conclude the proof as
\begin{equation*}
\|\pi(\cdot|\bm{s})-\pi'(\cdot|\bm{s})\|_{1}\leq\frac{\|r(\bm{s},\cdot)-r'(\bm{s},\cdot)\|_{\infty}}{\sqrt{2}\kappa}.
\end{equation*}
which concludes the proof.
\end{proof}

\section{Proof of Lemma~\ref{lemma2}}
\label{Appendix:lemma2}
\begin{proposition}
\label{bound}
The value functions sequence $V_{1}(\bm{s}),\ldots,V_{T}(\bm{s})$ generated by the Equ\eqref{TDupdaterule} satisfies
\begin{equation*}
\|\vL_{\hat{r}_{t}}^{\pi_{t}^{*}}(\cdot)-V_{t}(\cdot)\|_{\infty}
\leq
CC_{v}(t+1)^{C_{v}-1},
\end{equation*}
where $C=6\tau(2-C_{v}+\frac{1}{C_{v}}+\frac{1-C_{v}}{1+C_{v}})$, and $\pi_{t}^{*}=\argmax_{\pi\in\Pi}\rho_{\hat{r}_{t}}(\pi)$.
\end{proposition}

By Proposition 3 in \citet{YaoECML14} and Proposition 4.1 in \citet{Yu09}, we obtain the following result
\begin{align*}
\sum_{t=1}^{T}\left(\rho_{r_{t}}(\pi^{*})-\rho_{r_t}(\pi_{t})\right)
\leq
\sum_{t=1}^{T}\left(\rho_{r_{t}}(\pi_{t}^{*})-\rho_{r_{t}}(\pi_{t})\right)
\leq
\sum_{t=1}^{T}\frac{2-e^{-1/\tau}}{1-e^{-1/\tau}}\|\pi^{*}_{t}-\pi_{t}\|_{1}
\end{align*}
The result in Proposition~\ref{bound} leads the following inequalities
\begin{align*}
&\sum_{t=1}^{T}\left(\rho_{r_{t}}(\pi^{*})-\rho_{r_t}(\pi_{t})\right)\\
&\leq
\sum_{t=1}^{T}\frac{2-e^{-1/\tau}}{1-e^{-1/\tau}}(\|\pi_{t-1}^{*}-\pi_{t}\|_{1}+\|\pi_{t}^{*}-\pi_{t-1}^{*}\|_{1})\\
&\leq
\sum_{t=1}^{T}\frac{2-e^{-1/\tau}}{1-e^{-1/\tau}}\xi(\|\vL_{\hat{r}_{t-1}}^{\pi_{t-1}^{*}}-V_{t-1}\|_{\infty}+\frac{4\tau+2}{t})\\
&\leq
\frac{2-e^{-1/\tau}}{1-e^{-1/\tau}}\left(\frac{C\xi}{C_{v}}T^{C_{v}}+6\tau \xi\ln{T}+6\tau \xi\right).
\end{align*}

\section{Proof of Lemma~\ref{lemma3}}
\label{Appendix:lemma3}
The proof is following the same line as previous works\citep*{NIPS05,MOOR09,YaoECML14}, we rewrite the proof with our notations. By the definition of the expected average reward function, we have
\begin{align*}
&\sum_{t=1}^{T}\rho_{r_{t}}(\pi_{t})-\bE_{\pi_{t}}\left[\sum_{t-1}^{T}r_{t}(\bm{s}_{t},\bm{a}_{t})\right]\\
&=
\sum_{t=1}^{T}\sum_{\bm{s}\in S}\sum_{\bm{a}\in A}\left(d_{\pi_{t}}(\bm{s})\pi_{t}(\bm{a}|\bm{s})-d_{\aL,t}(\bm{s})\pi_{t}(\bm{a}|\bm{s})\right)r_{t}(\bm{s},\bm{a})\\
&\leq
\sum_{t=1}^{T}\sum_{\bm{s}\in S}|d_{\pi_{t}}(\bm{s})-d_{\aL,t}(\bm{s})|,
\end{align*}
where $d_{\aL,t}(\bm{s})$ is the state distribution at time step $t$ by following the policy sequence $\pi_{1},\ldots,\pi_{t}$ generated by the OMDP-PI algorithm. The last line can be obtain by $r_{t}(\bm{s},\bm{a})\in[0,1],\forall t=1,\ldots,T$.

For all $k=2,\ldots,t$, we have following results
\begin{align*}
&\|d_{\aL,k}-d_{\pi_{t}}\|_{1}\\
&=\|d_{\aL,k-1}P^{\pi_{k}}-d_{\pi_{t-1}}P^{\pi_{t}}\|_{1}\\
&\leq
\|d_{\aL,k-1}P^{\pi_{k}}-d_{\aL,k-1}P^{\pi_{t}}\|_{1}+\|d_{\aL,k-1}P^{\pi_{t}}-d_{\pi_{t-1}}P^{\pi_{t}}\|_{1}\\
&\leq
(\ln{(t-1)}-\ln{(k-1)})+
e^{-1/\tau}\|d_{\aL,k-1}-d_{\pi_{t-1}}\|_{1},
\end{align*}
Recurring the above result, we have
\begin{align*}
\|d_{\aL,t}-d_{\pi_{t}}\|_{1}
\leq&
\sum_{k=2}^{t}
(\ln{(t-1)}-\ln{(k-1)})e^{-(t-k)/\tau}
+e^{-t/\tau}\|d_{1}-d_{\pi_{t}}\|_{1}\\
\leq&
\left(1+\tau\right)\left(\frac{\tau^{2}}{t-1}+\tau e^{-(t-\tau-2)/\tau}\right)+2e^{-t/\tau},
\end{align*}
where the last inequality follows by
\begin{align*}
&\sum_{k=2}^{t}(\ln{(t-1)}-\ln{(k-1)})e^{-(t-k)/\tau}\\
&=
\int_{2}^{t}(\ln{(t-1)}-\ln{(k-1)})e^{-(t-k)/\tau}\mathrm{d}k
+\ln{(t-1)}e^{-(t-2)/\tau}\\
&=
\tau\int_{2}^{t}(\ln{(t-1)}-\ln{(k-1)})\frac{\mathrm{d}e^{-(t-k)/\tau}}{\mathrm{d}k}\mathrm{d}k+\ln{(t-1)}e^{-(t-2)/\tau}\\
&\leq
\tau\int_{2}^{t}\frac{e^{-(t-k)/\tau}}{k-1}\mathrm{d}k
=
\tau^{2}\int_{2}^{t}\frac{1}{k-1}\frac{\mathrm{d}e^{-(t-k)/\tau}}{\mathrm{d}k}\mathrm{d}k\\
&\leq
\frac{\tau^{2}}{t-1}+\tau^{2}\int_{2}^{t}\frac{e^{-(t-k)/\tau}}{(k-1)^{2}}\mathrm{d}k\\
&=
\frac{\tau^{2}}{t-1}+\tau^{2}\int_{\tau+2}^{t}\frac{e^{-(t-k)/\tau}}{(k-1)^{2}}\mathrm{d}k+\int_{2}^{\tau+{2}}\frac{e^{-(t-k)/\tau}}{(k-1)^{2}}\mathrm{d}k\\
&\leq
\frac{\tau^{2}}{t-1}+\frac{\tau^{2}}{\tau+1}\int_{2}^{t}\frac{e^{-(t-k)/\tau}}{k-1}\mathrm{d}k
+\int_{2}^{\tau+{2}}\frac{e^{-(t-k)/\tau}}{(k-1)^{2}}\mathrm{d}k.
\end{align*}
Hence, we have
\begin{equation*}
\int_{2}^{t}\frac{e^{-(t-k)/\tau}}{k-1}\mathrm{d}k\leq\left(1+\frac{1}{\tau}\right)\left(\frac{\tau^{2}}{t-1}+\tau e^{-(t-\tau-2)/\tau}\right).
\end{equation*}
The claimed result in Lemma~\ref{lemma3} can be obtained as
\begin{align*}
\sum_{t=1}^{T}\|d_{\aL,t}-d_{\pi_{t}}\|_{1}
\leq 2\tau^{3}\ln{T}+2\tau^{3}+2\tau^{3}e^{(\tau+2)}+2\tau.
\end{align*}

\section{Proof of Proposition~\ref{bound}}
\label{Appendix:bound}
\begin{proposition}
\label{contraction}
For arbitrary reward function $r(\bm{s},\bm{a})$, the corresponding value functions induced by two arbitrary policy $\pi_{1}$ and $\pi_{2}$ satisfy
\begin{equation*}
\|\vL_{r}^{\pi_{1}}-\vL_{r}^{\pi_{2}}\|_{\infty}\leq C_{\pi}\|\pi_{1}-\pi_{2}\|_{1}.
\end{equation*}
where $C_{\pi}$ is a positive constant. 
\end{proposition}

Let us define an auxiliary sequence of functions $\vL_{\hat{r}_{t}}^{\pi^{*}_{t}}(\bm{s}),t=1,\ldots,T$ which is defined as
\begin{equation*}
\vL_{\hat{r}_{t}}^{\pi^{*}_{t}}(\bm{s})=
\bE_{\pi^{*}_{t}}\left[\sum_{i=1}^{\infty}(\hat{r}_{t}(\bm{s},\bm{a})-\rho_{\hat{r}_{t}}(\pi^{*}_{t}))\right].
\end{equation*}
In above definition, $\pi^{*}_{t}$ is the optimal policy which satisfies
\begin{equation*}
\pi^{*}_{t}=\argmax_{\pi\in\Pi}\rho_{\hat{r}_{t}}(\pi),
\end{equation*}
and for all $\bm{s}\in S$, there is
\begin{equation*}
\vL_{\hat{r}_{t}}^{\pi_{t}^{*}}(\bm{s})\geq \vL_{\hat{r}_{t}}^{\pi}(\bm{s}),\forall \pi\in\Pi.
\end{equation*}
It is simple to verify that the value function is linear with respect to the reward function, i.e., $\vL_{\hat{r}_t}^{\pi}(\bm{s})=\frac{1}{t}\sum_{k=1}^{t}\vL_{r_{t}}^{\pi}(\bm{s})$. 
Hence, we can rewrite the sequence as
\begin{align*}
\vL_{\hat{r}_{t+1}}^{\pi_{t+1}^{*}}(\bm{s})=&
\vL_{\hat{r}_{t}}^{\pi_{t}^{*}}(\bm{s})
+\frac{1}{t+1}\left(\sum_{k=1}^{t+1}\vL_{r_{k}}^{\pi_{t+1}^{*}}(\bm{s})-\frac{t+1}{t}\sum_{k=1}^{t}\vL_{r_{k}}^{\pi_{t}^{*}}(\bm{s})\right)\\
=&
\vL_{\hat{r}_{t}}^{\pi_{t}^{*}}(\bm{s})
+\frac{1}{t+1}\vL_{r_{t+1}}^{\pi_{t+1}^{*}}(\bm{s})
-\frac{1}{t(t+1)}\sum_{k=1}^{t}\vL_{r_{k}}^{\pi_{t}^{*}}(\bm{s})
+\frac{1}{t+1}\left(\sum_{k=1}^{t}\vL_{r_{k}}^{\pi_{t+1}^{*}}(\bm{s})-\sum_{k=1}^{t}\vL_{r_{k}}^{\pi_{t}^{*}}(\bm{s})\right)\\
=&
(1-\frac{1}{t+1})\vL_{\hat{r}_{t}}^{\pi_{t}^{*}}(\bm{s})+\frac{1}{t+1}
\vL_{r_{t+1}}^{\pi_{t+1}^{*}}(\bm{s})
+\frac{t}{t+1}\left(\vL_{\hat{r}_{t}}^{\pi_{t+1}^{*}}(\bm{s})-
\vL_{\hat{r}_{t}}^{\pi_{t}^{*}}(\bm{s})\right)\\
\leq&
(1-\frac{1}{t+1})\vL_{\hat{r}_{t}}^{\pi_{t}^{*}}(\bm{s})+\frac{1}{t+1}
\vL_{r_{t+1}}^{\pi_{t+1}^{*}}(\bm{s}),
\end{align*}
where the last inequality can be obtained by the fact that $\pi^{*}_{t}$ is the optimal policy satisfies
\begin{equation*}
\vL_{\hat{r}_{t}}^{\pi_{t}^{*}}(\bm{s})\geq \vL_{\hat{r}_{t}}^{\pi_{t+1}^{*}}(\bm{s}),\forall \bm{s}\in S.
\end{equation*}
On the other hand, we can derive the lower bound as
\begin{align*}
\vL_{\hat{r}_{t+1}}^{\pi_{t+1}^{*}}(\bm{s})=&
\vL_{\hat{r}_{t}}^{\pi_{t}^{*}}(\bm{s})
+\frac{1}{t+1}\left(\sum_{k=1}^{t+1}\vL_{r_{k}}^{\pi_{t+1}^{*}}(\bm{s})-\frac{t+1}{t}\sum_{k=1}^{t}\vL_{r_{k}}^{\pi_{t}^{*}}(\bm{s})\right)\\
=&
\vL_{\hat{r}_{t}}^{\pi_{t}^{*}}(\bm{s})+
\frac{1}{t+1}\left(\sum_{k=1}^{t+1}\vL_{r_{k}}^{\pi_{t+1}^{*}}(\bm{s})-\frac{t+1}{t}\sum_{k=1}^{t+1}\vL_{r_{k}}^{\pi_{t}^{*}}(\bm{s})\right)+\frac{1}{t}\vL_{r_{t+1}}^{\pi_{t}^{*}}(\bm{s})\\
=&
\vL_{\hat{r}_{t}}^{\pi_{t}^{*}}(\bm{s})+
\frac{1}{t+1}\left(\sum_{k=1}^{t+1}\vL_{r_{k}}^{\pi_{t+1}^{*}}(\bm{s})-\sum_{k=1}^{t+1}\vL_{r_{k}}^{\pi_{t}^{*}}(\bm{s})\right)
-\frac{1}{t(t+1)}\sum_{k=1}^{t+1}\vL_{r_{k}}^{\pi_{t}^{*}}(\bm{s})+\frac{1}{t}\vL_{r_{t+1}}^{\pi_{t}^{*}}(\bm{s})\\
=&
\vL_{\hat{r}_{t}}^{\pi_{t}^{*}}(\bm{s})+
(\vL_{\hat{r}_{t+1}}^{\pi_{t+1}^{*}}(\bm{s})-\vL_{\hat{r}_{t+1}}^{\pi_{t}^{*}}(\bm{s}))
-\frac{1}{t+1}\vL_{\hat{r_{t}}}^{\pi_{t}^{*}}(\bm{s})
+\frac{1}{t+1}\vL_{r_{t+1}}^{\pi_{t}^{*}}(\bm{s})\\
\geq&
(1-\frac{1}{t+1})\vL_{\hat{r}_{t}}^{\pi_{t}^{*}}(\bm{s})
+\frac{1}{t+1}\vL_{r_{t+1}}^{\pi_{t}^{*}}(\bm{s}),
\end{align*}
where the last inequality comes from the fact $\pi_{t+1}^{*}$ is the optimal policy satisfies
\begin{equation*}
\vL_{\hat{r}_{t+1}}^{\pi_{t+1}^{*}}(\bm{s})\geq\vL_{\hat{r}_{t+1}}^{\pi_{t}^{*}}(\bm{s}),\forall \bm{s}\in S.
\end{equation*}
Then, we can obtain the following result
\begin{equation}
\label{rec}
|\vL_{\hat{r}_{t+1}}^{\pi_{t+1}^{*}}(\bm{s})-V_{t+1}(\bm{s})|
\leq
(1-\frac{1}{t+1})|\vL_{\hat{r}_{t}}^{\pi_{t}^{*}}(\bm{s})-V_{t}(\bm{s})|
+
\frac{1}{t+1}
\Delta_{t+1}.
\end{equation}
In above inequality, $\Delta_{t+1}=\max\{|\vL_{r_{t+1}}^{\pi^{*}_{t+1}}(\bm{s})-V_{r_{t+1}}^{\pi_{t+1}}(\bm{s})|,|\vL_{r_{t+1}}^{\pi^{*}_{t}}(\bm{s})-V_{r_{t+1}}^{\pi_{t+1}}(\bm{s})|\}$, which satisfies 
\begin{align*}
\Delta_{t+1}
\leq& C_{\pi}\max\{\|\pi_{t+1}^{*}-\pi_{t+1}\|_{1},\|\pi_{t}^{*}-\pi_{t+1}\|_{1}\}\\
\leq& C_{\pi}(\|\pi_{t}^{*}-\pi_{t+1}\|_{1}+\|\pi_{t}^{*}-\pi_{t+1}^{*}\|_{1})\\
\leq& C_{v}\|\vL_{\hat{r}_{t}}^{\pi_{t}^{*}}-V_{t}\|_{\infty}+\frac{(4\tau+2) C_{v}}{t+1}.
\end{align*}
The first term of the last inequality can be obtain by setting $C_{v}=\xi C_{\pi}$. The second part follows by the upper bound and the lower bound of $\vL_{
\hat{r}_{t}}^{\pi_{t+1}^{*}}$.
Next we show the bound of $\|\vL_{\hat{r}_{t}}^{\pi_{t}^{*}}-V_{t}\|_{\infty}$ by recurring Equ.\eqref{rec}
\begin{equation*}
\|\vL_{\hat{r}_{t}}^{\pi_{t}^{*}}-V_{t}\|_{\infty}
\leq
\frac{(4\tau+2) C_{v}}{t^{2}}
+\sum_{k=1}^{t-1}\frac{(4\tau+2) C_{v}}{k^{2}}\prod_{m=k+1}^{t}\left(1-\frac{1-C_{v}}{m}\right).
\end{equation*}
Let us take the logarithm of $\prod_{m=k+1}^{t}\left(1-\frac{1-C_{v}}{m}\right)$, there is
\begin{align*}
&\ln{\prod_{m=k+1}^{t}\left(1-\frac{1-C_{v}}{m}\right)}\\
&=
\sum_{m=k+1}^{t}\left(\ln{(m-1+C_v)}-\ln{m}\right)\\
&\leq
\sum_{m=k+1}^{t}\frac{-1+C_{v}}{m}\\
&\leq
-(1-C_{v})\int_{k+1}^{t+1}\frac{1}{m}\mathrm{d}m
=-(1-C_v)\ln{\frac{t+1}{k+1}},
\end{align*}
where the first inequality holds since the logarithm function is concave. Thus we derive the bound as
\begin{align*}
\|\vL_{\hat{r}_t}^{\pi_{t}^{*}}-V_{t}\|_{\infty}
\leq&
(4\tau+2) C_{v}\sum_{k=1}^{t}\frac{1}{k^{2}}\frac{(k+1)^{1-C_v}}{(t+1)^{1-C_v}}\\
\leq&
\frac{(4\tau+2) C_{v}}{(t+1)^{1-C_v}}\sum_{k=1}^{t}\frac{1}{k^{2}}\left(
k^{1-C_v}+(1-C_v)k^{-C_{v}}\right)\\
\leq&
\frac{(4\tau+2) C_v}{(t+1)^{1-C_v}}\left[2-C_v+\int_{1}^{t}k^{-C_v-1}\mathrm{d}k+(1-C_v)\int_{1}^{t}k^{-C_v-2}\mathrm{d}k\right]\\
=&
\frac{(4\tau+2) C_v}{(t+1)^{1-C_v}}\left[2-C_v+\frac{1}{C_{v}}-\frac{t^{-C_v}}{C_v}+\frac{1-C_v}{1+C_v}-\frac{1-C_v}{1+C_v}t^{-C_v-1}\right]\\
\leq&
CC_v(t+1)^{C_v-1},
\end{align*}
where $C=6\tau (2-C_v+\frac{1}{C_v}+\frac{1-C_v}{1+C_v})$. In above results, the second inequality follows by Taylor's theorem.

\section{Proof of Proposition~\ref{contraction}}
\label{Appendix:contraction}
Let us define the operator $\oT^{\pi}\vL_{r}^{\pi}(\bm{s})=\bE_{\pi}\left[r(\bm{s},\bm{a})-\rho_{r}(\pi)+\sum_{\bm{s}'\in S}p(\bm{s}'|\bm{s},\bm{a})\vL_{r}^{\pi}(\bm{s}')\right]$. we can obtain
\begin{align*}
&\vL_{r}^{\pi_{1}}(\bm{s})-\vL_{r}^{\pi_{2}}(\bm{s})\\
&=\oT^{\pi_{1}}\vL_{r}^{\pi_{1}}(\bm{s})-\oT^{\pi_{2}}\vL_{r}^{\pi_{2}}(\bm{s})\\
&=\left(\oT^{\pi_{1}}\vL_{r}^{\pi_{1}}(\bm{s})-\oT^{\pi_{2}}\vL_{r}^{\pi_{1}}(\bm{s})\right)+\left(\oT^{\pi_{2}}\vL_{r}^{\pi_{1}}(\bm{s})-\oT^{\pi_{2}}\vL_{r}^{\pi_{2}}(\bm{s})\right).
\end{align*}
By the definition of the operator, we rewrite the first term as
\begin{align*}
&\oT^{\pi_{1}}\vL_{r}^{\pi_{1}}(\bm{s})-\oT^{\pi_{2}}\vL_{r}^{\pi_{1}}(\bm{s})\\
&=\bE_{\pi_{1}}\left[r(\bm{s},\bm{a})-\rho_{r}(\pi_{1})+\sum_{\bm{s}'\in S}p(\bm{s}'|\bm{s},\bm{a})\vL_{r}^{\pi_{1}}(\bm{s}')\right]\\
&~~~-\bE_{\pi_{2}}\left[r(\bm{s},\bm{a})-\rho_{r}(\pi_{2})+\sum_{\bm{s}'\in S}p(\bm{s}'|\bm{s},\bm{a})\vL_{r}^{\pi_{1}}(\bm{s}')\right]\\
&=\left(\bE_{\pi_{1}}\left[Q_{r}^{\pi_{1}}(\bm{s},\bm{a})\right]-\bE_{\pi_{2}}\left[Q_{r}^{\pi_{1}}(\bm{s},\bm{a})\right]\right)+\left(\rho_{r}(\pi_{2})-\rho_{r}(\pi_{1})\right)\\
&=\left(Q_{r}^{\pi_{1}}(\bm{s},\pi_{1})-Q_{r}^{\pi_{1}}(\bm{s},\pi_{2})\right)
+\bE_{\bm{s}\sim d_{\pi_{2}}(\bm{s})}[Q_{r}^{\pi_{1}}(s,\pi_{2})-Q_{r}^{\pi_{1}}(s,\pi_{1})].
\end{align*}
The second term can be expressed as
\begin{align*}
&\oT^{\pi_{2}}\vL_{r}^{\pi_{1}}(\bm{s})-\oT^{\pi_{2}}\vL_{r}^{\pi_{2}}(\bm{s})\\
&\bE_{\pi_{2}}\left[r(\bm{s},\bm{a})-\rho_{r}(\pi_{2})+\sum_{\bm{s}'\in S}p(\bm{s}'|\bm{s},\bm{a})\vL_{r}^{\pi_{1}}(\bm{s}')-r(\bm{s},\bm{a})+\rho_{r}(\pi_{2})-\sum_{\bm{s}'\in S}p(\bm{s}'|\bm{s},\bm{a})\vL_{r}^{\pi_{2}}(\bm{s})\right]\\
&=\bE_{\bm{s}'\sim p^{\pi_{2}}(\bm{s'}|\bm{s})}[\vL_{r}^{\pi_{1}}(\bm{s}')-\vL_{r}^{\pi_{2}}(\bm{s}')].
\end{align*}
By summing up the above results, we obtain
\begin{align*}
&\vL_{r}^{\pi_{1}}(\bm{s})-\vL_{r}^{\pi_{2}}(\bm{s})\\
&=\left(Q_{r}^{\pi_{1}}(\bm{s},\pi_{1})-Q_{r}^{\pi_{1}}(\bm{s},\pi_{2})\right)
+\bE_{\bm{s}\sim d_{\pi_{2}}(\bm{s})}[Q_{r}^{\pi_{1}}(s,\pi_{2})-Q_{r}^{\pi_{1}}(s,\pi_{1})]\\
&~~~+
\bE_{\bm{s}'\sim p^{\pi_{2}}(\bm{s'}|\bm{s})}[\vL_{r}^{\pi_{1}}(\bm{s}')-\vL_{r}^{\pi_{2}}(\bm{s}')].
\end{align*}
In matrix notation, there is
\begin{equation*}
\vL_{r}^{\pi_{1}}-\vL_{r}^{\pi_{2}}=(Q_{r}^{\pi_{1},\pi_{1}}-Q_{r}^{\pi_{1},\pi_{2}})
-\bm{e}_{|S|}d_{\pi_{2}}^{\top}(Q_{r}^{\pi_{1},\pi_{1}}-Q_{r}^{\pi_{1},\pi_{2}})
+P^{\pi_{2}}(\vL_{r}^{\pi_{1}}-\vL_{r}^{\pi_{2}}),
\end{equation*}
where $\vL_{r}^{\pi}$ and $Q_{r}^{\pi,\pi'}$ are the length $|S|$ vectors whose $\bm{s}$th element is $\vL_{r}^{\pi}(\bm{s})$ and $Q_{r}^{\pi}(\bm{s},\pi')$, respectively. $\bm{e}_{|S|}$ denotes the length $|S|$ vector with all elements equal to $1$. $d_{\pi}$ is the $|S|$-dimensional vector whose $\bm{s}$th element is $d_{\pi}(\bm{s})$. $P^{\pi}$ is defined as the transition matrix induced by the policy $\pi$ and the transition $p(\bm{s}'|\bm{s},\bm{a})$. Thus, we obtain
\begin{equation*}
(\bm{I}_{|S|}-P^{\pi_{2}})(\vL_{r}^{\pi_{1}}-\vL_{r}^{\pi_{2}})=(\bm{I}_{|S|}-\bm{e}_{|S|}d_{\pi_{2}}^{\top})(Q_{r}^{\pi_{1},\pi_{1}}-Q_{r}^{\pi_{1},\pi_{2}}).
\end{equation*}
It is known that the Bellman equation with average reward function has no unique solution. However, the unique value function satisfies $d_{\pi}^{\top}\vL_{r}^{\pi}=0$. Hence, we add this condition to the above equation as
\begin{align*}
&(\bm{I}_{|S|}-P^{\pi_{2}})(\vL_{r}^{\pi_{1}}-\vL_{r}^{\pi_{2}})\\
&=(\bm{I}_{|S|}-\bm{e}_{|S|}d_{\pi_{2}}^{\top})(Q_{r}^{\pi_{1},\pi_{1}}-Q_{r}^{\pi_{1},\pi_{2}})-\bm{e}_{|S|}d_{\pi_{1}}^{\top}\vL_{r}^{\pi_{1}}+\bm{e}_{|S|}d_{\pi_{2}}^{\top}\vL_{r}^{\pi_{2}}\\
&=(\bm{I}_{|S|}-\bm{e}_{|S|}d_{\pi_{2}}^{\top})(Q_{r}^{\pi_{1},\pi_{1}}-Q_{r}^{\pi_{1},\pi_{2}})-\bm{e}_{|S|}d_{\pi_{2}}^{\top}(\vL_{r}^{\pi_{1}}-\vL_{r}^{\pi_{2}})
-(\bm{e}_{|S|}d_{\pi_{1}}^{\top}-\bm{e}_{|S|}d_{\pi_{2}}^{\top})\vL_{r}^{\pi_{1}}
\end{align*}
Then, by rearranging the above result:
\begin{equation*}
(\bm{I}_{|S|}-P^{\pi_{2}}+\bm{e}_{|S|}d_{\pi_{2}}^{\top})(\vL_{r}^{\pi_{1}}-\vL_{r}^{\pi_{2}})=
(Q_{r}^{\pi_{1},\pi_{1}}-Q_{r}^{\pi_{1},\pi_{2}})-(P^{\pi_{1}}_{sa}-P^{\pi_{2}}_{sa})Q_{r}^{\pi_{1}},
\end{equation*}
where $P^{\pi}_{sa}$ is the $|S|\times |A|$ matrix whose $(\bm{s},\bm{a})$th element is $d^{\pi}(\bm{s})\pi(\bm{a}|\bm{s})$. Using Proposition 12 in \citet{YaoECML14}, we obtain
\begin{align*}
\|\vL_{r}^{\pi_{1}}-\vL_{r}^{\pi_{2}}\|_{\infty}\leq
\frac{2-2e^{-1/\tau}}{1-e^{-1/\tau}}\|(\bm{I}_{|S|}-P^{\pi_{2}}+\bm{e}_{|S|}d_{\pi_{2}}^{\top})^{-1}Q_{r}^{\pi_{1}}\|_{\infty}\|\pi_{1}-\pi_{2}\|_{1},
\end{align*}
which concludes the proof by setting $\max_{\pi\in\Pi}\frac{2-2e^{-1/\tau}}{1-e^{-1/\tau}}\|(\bm{I}_{|S|}-P^{\pi}+\bm{e}_{|S|}d_{\pi}^{\top})^{-1}Q_{r}^{\pi}\|_{\infty}\leq C_{\pi}$.

\end{document}